%% file: main.tex
\documentclass{article}

\usepackage{times}
\usepackage{graphicx} 
\usepackage{subfigure} 
\usepackage{stackengine}

\usepackage{natbib}
\usepackage{algorithm}
\usepackage{algorithmic}

\usepackage[utf8]{inputenc} 
\usepackage[T1]{fontenc}    
\usepackage{hyperref}       
\usepackage{url}            
\usepackage{booktabs}       
\usepackage{amsfonts}       
\usepackage{nicefrac}       
\usepackage{microtype}      
\usepackage{amsmath,amsthm}
\usepackage{color}
\usepackage{xspace}

\usepackage{wrapfig}
\usepackage{subfigure} 
\usepackage{capt-of}
\usepackage{caption}

\usepackage{thm-restate}
\input{macros}

\usepackage[accepted]{icml2018}

\icmltitlerunning{A Distributed Second-Order Algorithm You Can Trust}

\begin{document} 

\twocolumn[
\icmltitle{A Distributed Second-Order Algorithm You Can Trust}




\begin{icmlauthorlist}
\icmlauthor{Celestine D{\"u}nner}{ibm}
\icmlauthor{Aurelien Lucchi}{eth}
\icmlauthor{Matilde Gargiani}{alu}
\icmlauthor{An Bian}{eth}
\icmlauthor{Thomas Hofmann}{eth}
\icmlauthor{Martin Jaggi}{epfl}
\end{icmlauthorlist}

\icmlaffiliation{ibm}{IBM Research -- Z{\"u}rich, Switzerland}
\icmlaffiliation{eth}{ETH, Z{\"u}rich, Switzerland}
\icmlaffiliation{epfl}{EPFL, Lausanne, Switzerland}
\icmlaffiliation{alu}{Albert-Ludwigs-Universit{\"a}t, Freiburg, Germany}

\icmlcorrespondingauthor{Celestine D{\"u}nner}{cdu@zurich.ibm.com}

\icmlkeywords{distributed learning}

\vskip 0.3in
]



\printAffiliationsAndNotice{}  

\begin{abstract} 
Due to the rapid growth of data and computational resources, distributed optimization has become an active research area in recent years. While first-order methods seem to dominate the field, second-order methods are nevertheless attractive as they potentially require fewer communication rounds to converge. However, there are significant drawbacks that impede their wide adoption, such as the computation and the communication of a large Hessian matrix. In this paper we present a new algorithm for distributed training of generalized linear models that only requires the computation of diagonal blocks of the Hessian matrix on the individual workers. To deal with this approximate information we propose an adaptive approach that - akin to trust-region methods - dynamically adapts the auxiliary model to compensate for modeling errors. We provide theoretical rates of convergence for a wide class of problems including $L_1$-regularized objectives. We also demonstrate that our approach achieves state-of-the-art results on multiple large benchmark datasets.

\end{abstract}

\input{01_introduction}
\input{02_method}

\input{02_implementation}
\input{03_analysis}
\input{04_related_work}

\input{05_experiments}
\input{06_conclusion}

\bibliography{references}
\bibliographystyle{icml2018}

\onecolumn
\clearpage
\input{appendix}

\end{document}

%% file: macros.tex
\DeclareMathOperator*{\argmin}{arg\,min}

\newcommand{\R}{\mathbb{R}}                      

\newcommand{\cI}{\mathcal{I}}

\newcommand{\diag}{\mathbf{diag}}

\newcommand{\xv}{ {\bf x}}

\newcommand{\uv}{ {\bf u}}
\newcommand{\vv}{ {\bf v}}
\newcommand{\wv}{ {\bf w}}

\newcommand{\sv}{ {\bf s}}

\newcommand{\alphav}{ {\boldsymbol \alpha}}
\newcommand{\betav}{ {\boldsymbol \beta}}

\newcommand{\0}{ {\bf 0}}

\newcommand{\bO}{F} 

\newcommand{\vvk}{{\vv_{[k]}}}

\newcommand{\Gap}{\mathcal{G}}

\theoremstyle{plain}
\newtheorem{theorem}{Theorem}
\newtheorem{lemma}[theorem]{Lemma}
\newtheorem{proposition}[theorem]{Proposition}

\newtheorem{assumption}{Assumption}

\newtheorem{remark}{Remark}
\newtheorem{corollary}[theorem]{Corollary}
\theoremstyle{definition}
\newtheorem{definition}{Definition}

\newcommand{\model}{\mathcal{M}}
\newcommand{\sigmat}{{\sigma_t}}
\newcommand{\msk}{{\model_\sigma^{(k)}}}
\newcommand{\mstk}{{\model_{\sigma_t}^{(k)}}}
\newcommand{\mst}{{\model_{\sigma_t}}}
\newcommand{\Dak}{{\Delta \alphav_{[k]}}}
\newcommand{\Dav}{{\Delta \alphav}}
\newcommand{\alphat}{{\alphav}^{(t)}}

\newcommand{\sigmasup}{{\sigma_{\text{sup}}}}

%% file: 01_introduction.tex

\section{Introduction}
\label{sec:intro}

The last decade has witnessed a growing number of successful machine learning applications in various fields, along with the availability of larger training datasets.
However, the speed at which training datasets grow in size is strongly outpacing the evolution of the computational power of single devices, as well as their memory capacity. Therefore, distributed approaches for training machine learning models have become tremendously important while also being  increasingly more accessible to users with the rise of cloud-computing.
Scaling up optimization algorithms for training machine learning models in such a setting poses many challenges. One key aspect is communication efficiency; because communication is often more expensive than local computation, the overall speed of distributed algorithms strongly depends on how frequently information is exchanged between workers.
In order to develop communication-efficient distributed algorithms, we advocate the use of second-order methods which benefit from faster rates of convergence compared to their first-order (gradient-based) counterparts, and hence require less communication rounds to achieve the same accuracy. However, second-order methods have the significant drawback of requiring the computation and storage -- and potentially the communication -- of a Hessian matrix. 
Exact methods are therefore elusive for large datasets and one has to resort to approximate methods. In this paper, we propose a method where every worker uses local Hessian information only (i.e.,  with respect to the local parameters on that worker), hence it does not require any second-order information to be communicated. Conceptually, this approach relies on approximating the full Hessian matrix with a block-diagonal version. At the same time, to automatically adapt to the model misfit, we use an adaptive approach similar in spirit to trust-region methods~\cite{conn2000trust}.
\vspace{-2mm}

\paragraph{Problem Setup \& Distributed Setting.}
We address the problem of training generalized linear models which are ubiquitous in machine learning, including e.g. logistic regression, support vector machines as well as sparse linear models such as lasso and elastic net. Formally, we address convex optimization problems with an objective of the form
\begin{equation}
    \label{eq:A} \quad \ \ 
    \bO(\alphav) := f(A\alphav )
    \ +\ \sum\nolimits_i g_i(\alpha_i),\ 
\end{equation}
where we assume $f$ to be smooth and convex, and $g_i$ to be convex functions. $A\in\R^{d\times n}$ is a given data matrix and $\alphav\in \R^n$ the parameter vector to be learned from data. 

We assume that every worker $k\in\{1\dots K\}$ only has access to its own local part of the data, which corresponds to a subset of the columns of the matrix $A$. In machine learning, these columns typically correspond to a subset of the features or data examples, depending on the application. For example, in the case where~\eqref{eq:A} corresponds to the objective of a regularized generalized linear model -- i.e., where $f$ is a data dependent loss and $g=\sum_i g_i$ a regularization term -- the columns of~$A$ correspond to features. In another scenario where \eqref{eq:A} corresponds to the dual representation of the respective problem, such as typically chosen for SVM models, the columns of $A$ correspond to data examples. 
\vspace{-2mm}

\paragraph{Block-separable model.}
In such a distributed setting we suggest optimizing a block-separable auxiliary model which can be split over workers. This auxiliary model is then updated in each round, upon receiving a summary of the updates from all workers. A significant advantage of such a model is that the workload of a single round can be parallelized across the individual workers, where each worker computes an update for its own model parameters by solving a local optimizaition task.
Then, to synchronize the work, each worker communicates this update to the master node which aggregates all the updates, applies them to the global model and shares this information with all the workers. One common problem faced with this type of distributed approach is to evaluate whether the local models can be trusted in order to update the global model. This is usually addressed by the selection of an appropriate step-size or by relying on a line-search approach. However, the latter uses a fixed model and typically requires multiple model evaluations which can therefore be computationally expensive. In this paper, we instead leverage ideas from trust-region methods~\cite{conn2000trust}, where we dynamically adapt the model based on how much we trust the approximate second-order information.
\vspace{-2mm}

\paragraph{Contributions.} We propose a new distributed Newton's method, built on an adaptive block-separable approximation of the objective function, and allowing the use of arbitrary solvers to solve the local subproblems approximately.
Two characteristics differentiate our approach from existing work. 
First, unlike previous methods that rely on fixed step-size schedules or line-search strategies, our algorithm evaluates the fit of the auxiliary model using a trust-region approach. This yields an efficient method with global convergence guarantees for convex functions, while providing full adaptivity to the quality of the second-order model.
Second, our method, to the best of our knowledge, is the first to give convergence guarantees for a distributed second-order method applied to problems with general regularizers (not necessarily strongly convex). This includes $L_1$-regularized objectives such as Lasso and sparse logistic regression as very important application cases, which were not covered by earlier methods such as \cite{shamir2014communication, zhang2015disco, wang2017giant, lee2017distributed}.

%% file: 02_method.tex

\section{Method Description}
\label{sec:method}
We present an iterative descent algorithm that minimizes the objective $\bO(\alphav)$ introduced in~\eqref{eq:A}. At each step, we optimize an auxiliary \emph{block-separable model} that acts as a surrogate for the objective $\bO(\alphav)$. This auxiliary model is adaptive and changes depending on its approximation quality.
\subsection{Block-Separable Model}
\label{sec:block_model}
Let us, in every iteration of our algorithm, consider the following auxiliary model replacing \eqref{eq:A}:
\begin{align} \model_\sigma( \Dav; \alphav) :=\hat f(A\alphav, A \Dav) + \sum_i g_i(\alpha_i+\Delta \alpha_i), 
\label{eq:modelCoCoA}
\end{align}
where $\hat f(A\alphav, A\Dav) $ is a second-order approximation of the data-dependent term in \eqref{eq:A}, i.e.,
\begin{align}
\hat f(A\alphav, A\Dav) :=\ & f(A\alphav) + \nabla f(A\alphav)^\top A\Dav\notag\\&+\frac \sigma {2}  \Dav^\top \tilde H(\alphav) \Dav. \label{eq:fapprox}
\end{align}
The parameter $\sigma\in \mathbb R_+$ is introduced to control the approximation quality of the auxiliary model; its role will be detailed in Section~\ref{sec:sigma}.

Let us consider \eqref{eq:fapprox} for the case where $\tilde H(\alphav)$ is chosen to be the Hessian matrix  $\nabla^2_\alphav f(A\alphav)$. Then, the auxiliary model~\eqref{eq:modelCoCoA} with $\sigma=1$ corresponds to a classical second-order approximation of the function $f$. 
However,  this choice of $\tilde H$ is not feasible in a distributed setting where the data is partitioned among the workers, since the computation of the Hessian matrix requires access to the entire data matrix.
\vspace{-1mm}

\paragraph{Partitioning.}
In particular, we assume each worker has access to a subset $\cI_k$ of the columns of $A$. In our setting, $\cI_k$ are disjoint index sets such that
$ \bigcup\nolimits_k \cI_k = [n], \;\;\cI_i \cap \cI_j = \emptyset \;\; \forall i\neq j $
and $n_k := \left|\cI_k\right|$ denotes the size of partition $k$. Hence, each machine stores in its memory the submatrix $A_{[k]}\in\R^{d \times n_k}$ corresponding to its partition $\cI_k$.

Given such a partitioning, we suggest choosing $\tilde H$ to be a block diagonal approximation to the Hessian matrix~$\nabla^2_\alphav f(A\alphav)$ aligned with the partitioning of the model parameters, such that
\begin{equation}
\Dav^\top \tilde H(\alphav) \Dav = \sum\nolimits_k \Dak^\top \tilde H(\alphav) \Dak.  \label{eq:Htildesep}
\end{equation}
We use the notation $\uv_{[k]}$ to denote the vector $\uv$ with only non-zero coordinates for $i\in \cI_k$.
As a consequence of \eqref{eq:Htildesep} the model presented in~\eqref{eq:modelCoCoA}  splits over the $K$ partitions, i.e., 
\begin{equation}
\model_{\sigma}(\Dav; \alphav) = \sum\nolimits_k \msk(\Dak; \alphav),
\end{equation}
where each  subproblem $\msk(\Dak; \alphav)$ only requires access to the local data indexed by $\cI_k$, the respective coordinates of the model $\alphav$, as well as
$\vv:= A\alphav$:
\begin{align}  \msk(\Dak; \alphav)  :=& \frac 1 K f(\vv) +\nabla f(\vv)^\top A\Dak\notag \\&+ \frac {\sigma} 2 \Dak^\top \tilde H(\alphav) \Dak \notag\\& \sum\nolimits_{i\in\cI_k} g_i((\alphav+\Dak)_i).
\label{eq:localModel}
\end{align}
Hence, in a distributed setting, each worker is assigned the subproblem corresponding to its partition. These individual subproblems can be optimized independently and in parallel on the different workers. We note that this requires access to the shared information $\vv$ on every node; we will detail in Section \ref{sec:implementation} how this can  be efficiently  achieved in a distributed setting. A significant benefit of this model is that it is based on local second-order information and does not require sending gradients and Hessian matrices to the master node, which would be a significant cost in terms of communication.

\subsection{Approximation Quality of the Model}
\label{sec:sigma}
The role of the $\sigma$ parameter introduced in \eqref{eq:modelCoCoA} is to account for the loss of information that arises by enforcing the approximate Hessian matrix of $f$ to have a block diagonal structure. The better the approximation, the closer to $1$ the optimal $\sigma$ parameter is. If the Hessian approximation is unreliable, then the model should be adapted accordingly by changing the value of $\sigma$.
An alternative model to~\eqref{eq:modelCoCoA} would be to include a damping factor to the second-order term, i.e.,  use $\frac{\sigma}{2} \Dav^\top \tilde H(\alphav) \Dav + \sigma' \| \Dav \|^2 $ where $\sigma' > 0$. This type of model is usually employed in trust-region methods~\cite{conn2000trust} where $\sigma=1$, and $\sigma' > 0$ is chosen to ensure strong-convexity. The use of $\sigma' > 0$ might therefore not be necessary for models that are already (strongly)-convex. We conducted a set of experiments to determine whether this alternative model would achieve better empirical performance and we found little difference between the two models. We will therefore report results for our suggested model with $\sigma' = 0$ in the experimental section.
\vspace{-1mm}

\paragraph{Adaptive Choice of $\sigma$.} We have established that the $\sigma$ parameter has a central role for the convergence and the practical performance of our method, and we therefore need an efficient way to choose and update this parameter in an adaptive manner. Here we suggest updating $\sigma$ at each iteration of the algorithm using an update rule inspired by trust-region methods~\cite{Cartis2011}, where $\sigma$ acts as the reciprocal of the trust-region radius. Further details are provided in Section~\ref{sec:algo}.

\begin{figure}[t]
\includegraphics[width=\columnwidth]{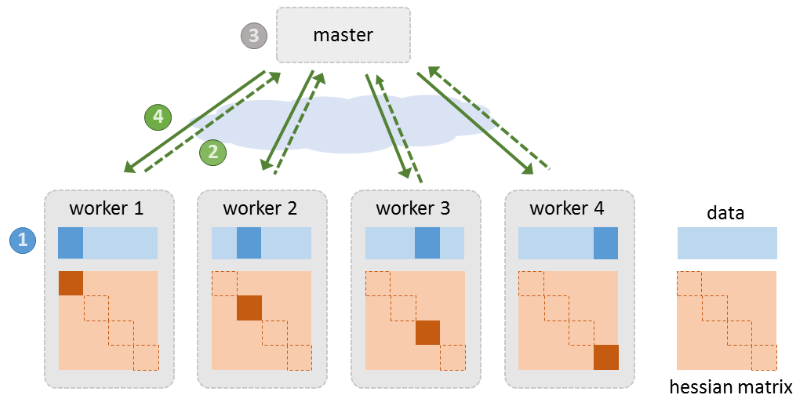}
\caption{Four-stage algorithmic procedure of ADN. Every worker only has access to its local partition of the data matrix and the respective block of the Hessian matrix. Arrows indicate the (synchronous) communication per round.\vspace{-3mm}}
\centering
\label{fig:cluster}
\end{figure}

\subsection{Algorithm Procedure}
\label{sec:algo}
The pseudo-code of the proposed approach, denoted as Adaptive Distributed Newton method (ADN), is summarized  in Algorithm~\ref{alg:adaptive_cocoa} and the four-stage  iterative procedure is illustrated in Figure \ref{fig:cluster}. We focus on a master-worker setting in this paper, but our algorithm could similarly be applied in a non-centralized fashion. Specifically, in every round, each worker $k$ works on its local subproblem \eqref{eq:localModel} to find an update~$\Dak$ to its local parameters of the model (stage~1). Then, it communicates this update to the master node (stage~2) which, aggregates the updates, and decides a new $\sigma_{t+1}$ for the next iteration based on the misfit of the current model (stage~3). Finally, the master node broadcasts the new model together with $\sigma_{t+1}$ to every worker (stage~4) for the next round.
%
Note that in Algorithm \ref{alg:adaptive_cocoa} we have not explicitly stated the communication of $\vv$ and two scalars that are necessary for evaluating the function values distributedly; we will elaborate more on this in Section \ref{sec:implementation}.
\vspace{-1mm}

\paragraph{Local Solver.} The computation of the model update~$\Dak$ on every worker (stage 2 of our algorithm) can be done using any arbitrary solver, depending on user preference or the available hardware resources. As in~\cite{Smith:2016wp}, the amount of computation time spent in the local solver is a tunable hyperparameter. This allows the algorithm to be optimally adjusted according to the trade-off between communication and computation cost of a given system. To reflect this flexibility in our theory we will assume the local subproblems \eqref{eq:localModel} are not necessarily optimized exactly but $\eta$-approximately, i.e., the local updates $\Dak$ are such that:
\begin{align}
\frac{\msk(\Dak; \alphav)- \msk(\Dav^\star_{[k]}; \alphav)}{\msk(\0; \alphav)-\msk(\Dav^\star_{[k]}; \alphav)}\leq \eta\, ,
\label{eq:theta_approx}
\end{align}
where $\Dav^\star_{[k]} := \argmin_\Dak \msk(\Dak; \alphav)$.\footnote{%
Note that the notion $\eta\in[0,1)$ 
of multiplicative subproblem accuracy is sensible even in the case when the block-wise Hessian matrix $\tilde H$ is not necessarily positive definite as long as $g$ is strongly convex or of bounded support.%
}

As previously mentioned, one of the key steps in the adaptive approach presented in Algorithm~\ref{alg:adaptive_cocoa} is the strategy for adapting the model over iterations. This is done by adjusting~$\sigma$ in every iteration $t$ resulting in a schedule described by the sequence $\{\sigma_t\}_{t\geq0}$. In particular, after every iteration, we adjust $\sigma_t$ based on the agreement between the model function \eqref{eq:modelCoCoA} and the objective \eqref{eq:A} for the current iterate. This is measured by the variable $\rho_t$ defined in~\eqref{eq:rho} in Algorithm~\ref{alg:adaptive_cocoa}. If~$\rho_t$ is close to $1$ there is a good agreement between the model $\model_{\sigma_t}(\cdot)$ and the function $\bO(\cdot)$ and we retain our current model. On the other hand, if the model over-estimates the objective, we decrease $\sigma$ for the next iteration, which can be thought of as adjusting the trust in the current approximation of the Hessian. On the contrary, if our model under-estimates the objective we increase $\sigma$.
In addition, we only apply updates $\Dav$ that satisfy $\rho_t\geq\xi$ and hence provide sufficient function decrease. If this is not fulfilled,  the step is rejected and a new update is computed in the next iteration, based on the adjusted model.
In order to adequately deal with all these cases that influence~$\sigma$, we introduce two constants $\zeta$ and $\gamma$ that control how to update~$\sigma$ based on the value of~$\rho_t$ (see \eqref{eq:sigma_update} in Algorithm~\ref{alg:adaptive_cocoa}). We will discuss the choice of these constants in the experiment section.

 \begin{algorithm}[h!]
   \caption{Adaptive Distributed Newton Method (ADN)}
   \label{alg:adaptive_cocoa}
\begin{algorithmic}[1] 
   \STATE {\bfseries Input:}  
   $\alphav_0 \in \R^n$ (e.g., ~$\alphav_0 \!=\! {\bf 0}$) and $\sigma_0>0$.\\
   $\quad\gamma, \zeta>1$,  $\tfrac 1 \zeta >\xi>0$ and $\eta\in [0,1)$
   \FOR{$t=0,1,\dots,\text{until convergence}$}
     \FOR{$k\in[K] \text{ in parallel}$}
   \STATE Obtain $\Dak$ by minimizing $\mstk(\Dak; \alphat)$ $\eta$-approximately
	\ENDFOR
	\STATE Aggregate updates $\Dav = \sum_k\Dak$  
   \STATE Compute $\bO(\alphav^{(t)}+ \Dav)$ (distributed over workers)
   \STATE Compute $\model_\sigmat(\Dav;\alphat)$ (distributed over workers)
   \STATE Evaluate\vspace{-2mm}
	\begin{equation}
	\rho_t:=\dfrac{\bO(\alphav^{(t)})-\bO(\alphav^{(t)}+ \Dav)}{\bO(\alphav^{(t)})-\model_{\sigma_t}(\Dav; \alphat)}
	\label{eq:rho}
	\end{equation}
	\STATE Set\vspace{-4mm}
	\begin{equation}
	\alphav^{(t+1)} := \begin{cases}
	\alphav^{(t)} + \Dav & \text{ if } \rho_t \geq \xi \\
	\alphav^{(t)} & \text{ otherwise}
	\end{cases}
	\end{equation}
	\STATE Set\vspace{-3mm}
	\begin{equation} \label{eq:sigma_update}
	\hspace{-2mm}
	\sigma_{t+1}:= \begin{cases}
	\frac 1\gamma  \sigma_t & \text{ if } \rho_t> \zeta \text{ (too conservative) }\\
	\;\;\sigma_t & \text{ if } \zeta \geq \rho_t\geq \frac 1 \zeta \text{ (good fit)}\\
	\gamma \sigma_t& \text{ if } \frac 1 \zeta>\rho_t \text{ (too aggressive)} 
	\end{cases}\vspace{-2mm}
	\end{equation}
   \ENDFOR
\end{algorithmic}
\end{algorithm}

%% file: 02_implementation.tex

\section{Implementation}
\label{sec:implementation}
In order to implement Algorithm \ref{alg:adaptive_cocoa} efficiently in a distributed environment, two key aspects need to be considered.

\subsection{Shared Information}
We have seen in Section \ref{sec:block_model} that every worker needs access to $\vv:= A\alphav$ in order to evaluate the gradient $\nabla f(A\alphav)$ for solving the local subproblem. To avoid the evaluation of $\vv$ in every round we suggest sharing and updating the vector $\vv=A\alphav$ throughout the algorithm -- thus, the term shared vector. Hence, if the model parameters are updated locally, the respective change $\Delta \vv_{[k]} = A\Dak$ is shared between workers, whereas the local model parameters $\alphav_{[k]}$ are kept local on every worker.  A similar approach to achieve communication-efficiency is suggested in \cite{Smith:2016wp}. They also emphasize that the vector to be communicated is $d$-dimensional which can be preferable compared to the $n$-dimensional model vector $\alphav$, depending on the dimensionality of the problem. This shared vector modification is a minor change of step 6 in Algorithm \ref{alg:adaptive_cocoa}, where $\Delta \vv = \sum_k \Delta \vvk$ is aggregated and shared instead of~$\Dav$.

\subsection{Communication-Efficient Function Evaluation}
Let us  detail how $\rho_t$ in Step 9 of Algorithm \ref{alg:adaptive_cocoa} can be evaluated efficiently without central access to the model $\alphav$.  We therefore consider the individual terms in \eqref{eq:rho} separately: The cost $\bO(\alphav)$ is known from the previous iteration and can be stored in memory. The cost at the new iterate $\bO(\alphav+\Dav) =f(A(\alphav+\Dav))+\sum_i g_i((\alphav+\Dav)_i)$
is composed of two terms, where the first term can be computed on the master locally as $f(\vv + \sum_k \Delta \vv_k)$ and the second term needs to  be computed in a distributed fashion. Every node computes 
 $g_{(k)} := \sum_{i\in \mathcal I_k} g_i((\alphav + \Dak)_i)$
 based on its local model parameters and sends the resulting value to the master node,  which adds the overall sum to the first term, completing the evaluation of the new objective value. Similarly, the model cost $\model_\sigmat( \Dav;\alphav)$ is computed distributedly by every node independently evaluating $\model_\sigmat^{(k)} (\Dak; \alphav_{[k]})$ and then sharing the result.
 Note that this step can be computationally expensive, since it requires one pass through the local data on every node; the communication cost of the two scalar values is negligible.
 

%% file: 03_analysis.tex

\section{Convergence Analysis}
\label{sec:analysis}

We now establish the convergence of Algorithm~\ref{alg:adaptive_cocoa} for the general class of functions fitting~\eqref{eq:A}.
\begin{theorem}[non-strongly convex $g_i$]
\label{th:main_convergence}
Let $f$ be $\tfrac 1 \tau$-smooth and $g_i$ be convex functions. Assume the sequence $\{\sigma_t\}_{t\geq0}$ is bounded by~$\sigma_\text{sup}$.\footnote{We will theoretically establish the upper bound $\sigma_\text{sup}$ for two general scenarios in Appendix A.7.} Then, Algorithm \ref{alg:adaptive_cocoa} reaches a suboptimality $\bO(\alphat)-\bO(\alphav^\star) \leq \varepsilon$ within a total number of
$$
\frac {1}{\log(\gamma)} \log\Big(\frac {\sigma_\text{sup}}{\sigma_0}\Big) +\frac {2 } {\varepsilon} C_1 \sigma_{sup}
$$
iterations, where $C_1>0$ is a constant defined as $C_1:=\frac{2 ( 4 L^2 R^2+ \tau \varepsilon_0)}{\tau \xi (1-\eta)}$  where  $L,R>0$ are such  that  $|\alpha_i^{(t)}|<L \;\; \forall i, t\geq 0$ and $\|A_{:,i}\|<R\;\forall i$, and $\varepsilon_0 :=\bO(\alphav_{0})-\bO(\alphav^\star)$ is the initial suboptimality. 
\end{theorem}

For the special case where $g_i$ are strongly-convex, Algorithm~\ref{alg:adaptive_cocoa} achieves a faster rate of convergence as described in the following theorem.
\begin{theorem}[strongly-convex $g_i$]
\label{th:main_convergence_strongly_convex}
Let $f$ be $\tfrac 1 \tau$-smooth and $g_i$ $\mu$-strongly convex. Assume the sequence $\{\sigma_t\}_{t\geq0}$ is bounded by $\sigma_\text{sup}$. 
Then, Algorithm \ref{alg:adaptive_cocoa} reaches a suboptimality $\bO(\alphat)-\bO(\alphav^\star) \leq \varepsilon$ within a total number of 
$$
\frac {1}{\log(\gamma)} \log\left(\gamma \frac {\sigma_\text{sup}}{\sigma_0}\right) + \frac 2 {\log(C_2^{-1})} \log\left(\frac{\varepsilon_0}{\varepsilon}\right)
$$
iterations, where $C_2\in (0,1)$ is a constant defined as $C_2:= 1- \xi(1-\eta) \frac{\mu\tau}{c_A \sigma_{\sup} +\mu\tau}$ with $c_{A}=\max_k \|A_{[k]}\|^2$ and $\varepsilon_0=\bO(\alphav_{0})-\bO(\alphav^\star)$ measures the initial suboptimality.
\end{theorem}
Note that for strongly-convex functions $g_i$, similar global rates of convergence to the one derived in Theorem~\ref{th:main_convergence_strongly_convex} are obtained by existing distributed second-order methods such as~\cite{lee2017distributed, wang2017giant}. However, we are not aware of any result similar to Theorem~\ref{th:main_convergence} in the more general case where $g_i$ are non-strongly convex functions.


\textbf{Proof Sketch}

We summarize the main steps in the proof of Theorem~\ref{th:main_convergence} and~\ref{th:main_convergence_strongly_convex}, a detailed derivation is provided in the Appendix.

\textbf{Step 1.} Recall that the model with  block diagonal Hessian approximation, described in Section~\ref{sec:block_model}, acts as a surrogate to minimize the function introduced in \eqref{eq:A}. The first step  is therefore to establish a bound on the decrease of the auxiliary model  for every step of the algorithm, given that each local subproblem is solved $\eta$-approximately.
This bound on the model decrease $\model_\sigmat(\0;\alphav)-\model_\sigmat(\Dav;\alphav)$, stated in Lemma \ref{lem:dec},  is established using a primal-dual perspective on the problem, similar to \cite{sdca}.

\begin{restatable}[]{lemma}{lemmathree}
 \label{lem:dec}
Assume $f$ is $\tfrac 1 \tau$-smooth and $g_i$ are $\mu$-strongly convex with $\mu\geq 0$. Then, the per-step model decrease of Algorithm~\ref{alg:adaptive_cocoa} can be lower bounded  as:
\begin{align*}
&\model_\sigmat(\0; \alphat) - \model_\sigmat(\Dav; \alphat) \\
&\quad\geq (1-\eta) \left[\kappa \Gap(\alphav^{(t)})  - \frac {\kappa^2} 2 R^{(t)}\right], 
\end{align*}
where $\Gap(\alphav^{(t)})$ denotes the duality gap, $\kappa\in(0,1]$ and
\begin{align*}R^{(t)}:=& \sigmat (\uv^{(t)}-\alphav^{(t)})^\top  \tilde H(\alphav) (\uv^{(t)}-\alphav^{(t)})\\&- \tfrac {\mu (1-\kappa)}{\kappa} \|\alphav^{(t)}-\uv^{(t)}\|_2^2
\end{align*}
with $u_i^{(t)}\in \partial g^*_i(A_{:,i}^\top\nabla f(A \alphav^{(t)}))$\footnote{$g^*_i$ denotes the convex conjugate of the function $g_i$, which 
	is defined as $g^*_i(u):= \sup_v [uv - g_i(v)]$.}.
\end{restatable} 

\textbf{Step 2.} For iterations that are successful (i.e., they provide sufficient function decrease as measured by $\rho_t\geq\xi$ in step~10 of Algorithm~\ref{alg:adaptive_cocoa}), the construction of  Algorithm \ref{alg:adaptive_cocoa} allows us to relate the model decrease from Lemma \ref{lem:dec} to the function decrease $\bO(\alphat)-\bO(\alphat+\Dav)$ through the parameter $\xi$. This yields a lower bound on the function decrease for every successful update as provided in Lemma~\ref{lem:f_dec} below.

\begin{restatable}[]{lemma}{lemmafour}
\label{lem:f_dec}
The function decrease of Algorithm \ref{alg:adaptive_cocoa} for a successful update  $(\Dav,\sigmat)$ can be bounded  as:
\[\bO(\alphat) - \bO(\alphat+\Dav) \geq \xi (1-\eta) \left[\kappa\Gap(\alphav^{(t)})  - \frac {\kappa^2} 2 R^{(t)}\right], \]
where $\kappa\in(0,1]$ and $R^{(t)}$ is defined as in Lemma \ref{lem:dec}.
\end{restatable}

\textbf{Step 3.} At this stage, we have shown that each successful iteration decreases the function value,
 therefore making progress towards the optimum. However, unsuccessful iterations (for which $\rho_t < \xi$) do not decrease the objective  and overall convergence to an optimum can only occur if the number of these iterations is limited. The next step is therefore to bound the number of  unsuccessful iterations. This is accomplished by showing that the construction of the sequence $\{\sigma_t\}_{t\geq 0}$ is such that the number of successive unsuccessful iterations is bounded and, hence, increasing $\sigma$ will eventually yield a successful iteration that will allow us to decrease the objective function. This results in a bound on the number of successful and unsuccessful iterations derived in the Appendix. Finally, the rate of convergence in Theorem \ref{th:main_convergence} and Theorem \ref{th:main_convergence_strongly_convex} are obtained by combining the bound on the number of steps with the function decrease for each successful step.

\textbf{Remark.}  Note that the update scheme \eqref{eq:sigma_update} in Algorithm~\ref{alg:adaptive_cocoa} is one of many that satisfy the conditions required for proving convergence. For further details, we refer the reader to the literature on trust-region methods~\cite{conn2000trust}.

%% file: 04_related_work.tex

\section{Related Work}

\textbf{First-order Methods.} 
Most first-order stochastic methods require frequent communication which comes with high costs in distributed settings, thus they are often prefered in multi-core settings. This is for example the case for the popular Hogwild! algorithm~\cite{Niu:2011wo} that relies on asynchronous SGD updates in a lock-free setting and requires communication after each optimization step.
Alternatives include variance-reduced methods such as~\cite{lee2015distributed} and coordinate descent methods such as~\cite{richtarik2013distributed}, however, they suffer similar communication bottlenecks. 

\textbf{Trust-region Methods.} 
\quad
These methods use a surrogate model  to approximate the objective within a region around the current iterate. The size of the trust region is expanded or contracted according to the fitness of the surrogate model to the true objective. For efficiency reasons, the surrogate  model is often a quadratic model~\cite{conn2000trust,karimireddy2018newton}, although cubic models can also be used~\cite{nesterov2006cubic}. 
Though trust-region methods have been extensively used in a single-machine setting, to the best of our knowledge we are the first to apply a trust-region-like approach in a \textit{distributed} setting. 

\textbf{Line-search vs Trust-region.}
Line-search techniques are a popular way to guarantee convergence and they have recently been explored in distributed settings, e.g., \cite{Hsieh:2016wg,lee2017distributed,Trofimov:2017ho,mahajan2017distributed,lee2018distributed}.  Our trust-region approach has clear advantages compared to line-search methods: i) a line-search method  assumes a \textit{fixed} auxiliary model --which may be an arbitrarily bad approximation of the true objective-- that is used to find an acceptable step size.  In contrast, our approach  adaptively tunes the auxiliary model to ensure that it is a good fit to the true objective. 
ii) in general,  a line-search method requires multiple objective value evaluations in order to test different step sizes, while our approach only needs one objective value evaluation to calculate $\rho_t$. The advantages of our method are verified empirically in Section~\ref{sec:experiments}.

\textbf{Approximate Newton-type Methods.}
For distributed $L_1$-regularized problems~\cite{Andrew:2007cu} proposed a quasi-newton method without convergence guarantees. Most of the literature on Newton-type methods are otherwise designed to optimize strongly-convex objectives.
DANE~\cite{shamir2014communication} is a distributed approximate Newton-method with a linear rate of convergence for quadratic functions. AIDE~\cite{reddi2016aide} is an accelerated version using the Catalyst scheme. Another similar  approach is DiSCO~\cite{zhang2015disco} which consists of an inexact damped Newton method using conjugate gradient steps, achieving a linear rate of convergence for self-concordant functions. 
Finally, GIANT~\cite{wang2017giant} relies on conjugate gradient steps and achieves a local linear-quadratic convergence rate but does not provide a global rate of convergence. It was shown empirically to outperform DANE, AIDE and DiSCO. 
Note that the convergence results of these approaches require each subproblem to be solved with high accuracy, which is often prohibitive for large-scale datasets. 
%
Some approaches suggest using a block-diagonal Hessian approximation such as \cite{Hsieh:2016wg
,lee2017distributed,lee2018inexact} but they all 
rely on a line-search approach which is shown to be inferior to our adaptive approach in the experimental section.
While both our approach and~\cite{lee2017distributed} require $\mathcal{O}(\log(1/\varepsilon))$ iterations to reach $\varepsilon$ accuracy for a strongly-convex $g$, we further  provide a rate of convergence for the more general case where $g$ is non-strongly  convex.

\textbf{Distributed Primal-Dual Methods.}
Approaches such as~\citep{yang2013trading,Jaggi:2014vi,zhang2015disco,zheng2017general,wang2017giant} are restricted to strongly-convex regularizers, and typically work on the dual formulation of the objective. CoCoA~\citep{Smith:2016wp} provides an extension to a wider class of regularizers, including $L_1$, as of interest here. Although it allows for the use of arbitrary solvers on each worker to regulate the amount of communication, this approach is inherently based on a first-order model of the objective and does not use second-order information.


In an earlier work by~\cite{gargiani2017hessian} a modification of CoCoA was discussed which incorporates local second-order information for the general class of problems \eqref{eq:A}. We here extend this approach  to be adaptive to the quality of the local surrogate model in a trust region sense, in contrast to using fixed Hessian information \cite{Hsieh:2016wg,gargiani2017hessian,lee2017distributed,lee2018inexact}.

%% file: 05_experiments.tex

\section{Experimental Results}
\label{sec:experiments}
We devote the first part of this section to analysing the properties of our adaptive scheme. In the second part we evaluate its performance for training a logistic regression model regularized with $L_1$ and $L_2$ regularization. We compare ADN to state-of-the-art distributed solvers on four large-scale datasets (see Table~\ref{tbl:datasets}). All algorithms presented in this section are implemented in C++, they are optimized for sparse data structures and use MPI to handle communication between workers. If not stated otherwise, we use $K=8$ workers.

 \begin{table}[h]
\centering
\small{
\centering
 \begin{tabular}{|l| r r r|} 
\hline
  & \# examples & \# features & sparsity \\ [0.1ex] 
  \hline\vspace{0.3ex}
url&2'396'130& 3'230'442&3.58 E-05\\ [0.1ex] 
webspam&262'938 &680'715&2.24 E-04\\[0.1ex] 
 kdda & 8'407'751&19'306'083&1.80 E-06 \\[0.1ex] 
 criteo &45'840'617 & 1'000'000&1.95 E-06\\[0.1ex] 
 \hline
 \end{tabular}}
 \caption{Datasets used for the experiments.}
 \label{tbl:datasets}
\vspace{-0.3cm}
\end{table}
 \begin{figure}[t]
 \centering
 \includegraphics[height=4.2cm]{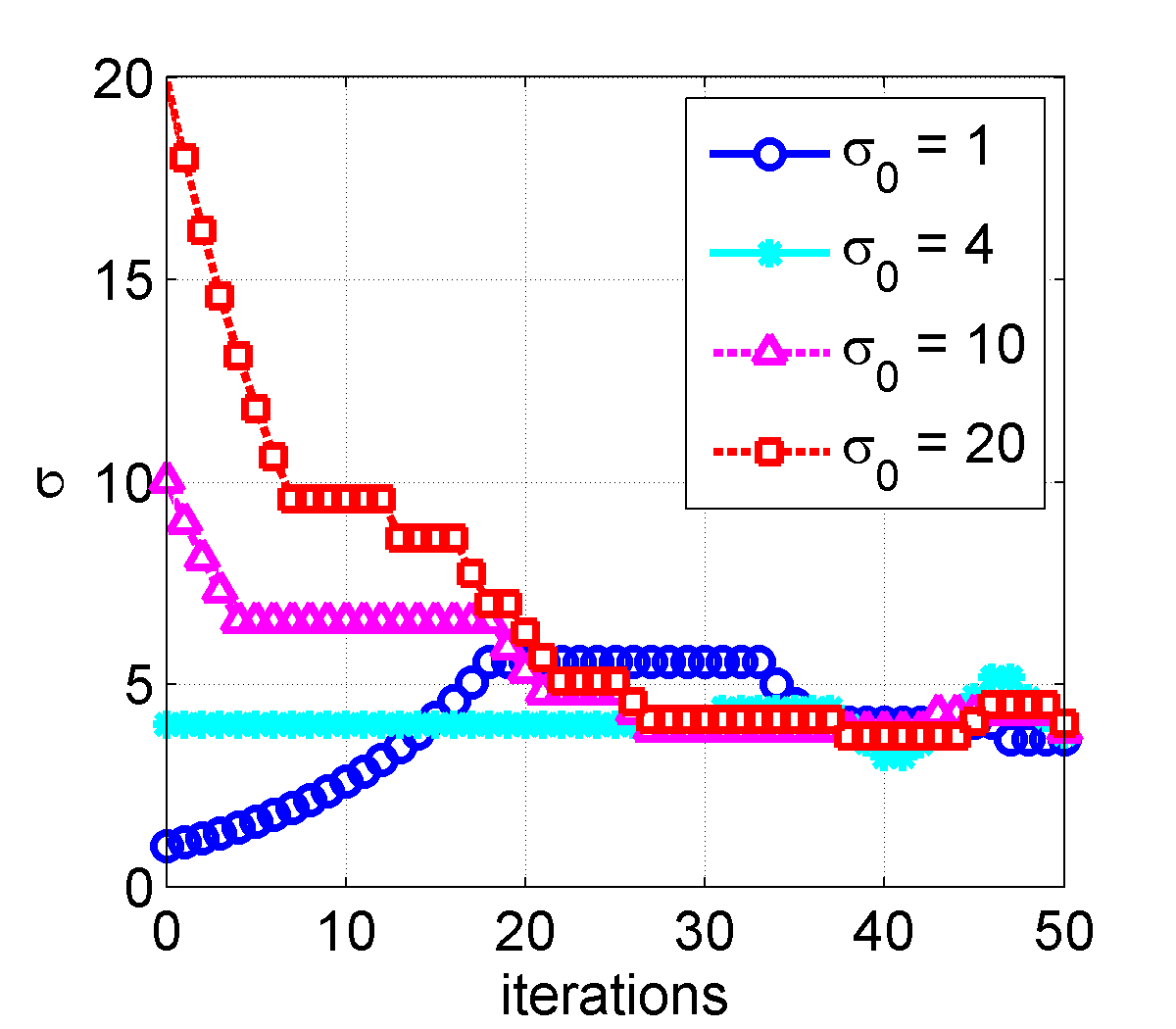}\vspace{-0.1cm}
 \caption{Robustness to initialization: Training the dual logistic regression model on a subsample (1 million examples) of the criteo dataset for different $\sigma_0$ with $\gamma =1.1$, $\zeta =1.1$ and $\xi=0$.\vspace{-0.2cm}}
 \label{fig:init}
 \end{figure}

\subsection{Algorithm Properties}

\paragraph{Initialization of $\sigma$.}
Given the wide dissemination of machine learning models to diverse fields, it is becoming increasingly important to develop algorithms that can be deployed without requiring expert knowledge to choose parameters. In this context we first check the sensitivity of our algorithm to the choice of $\sigma_0$. The results shown in Figure~\ref{fig:init} demonstrate that our adaptive scheme dynamically finds an appropriate value of $\sigmat$, independently of the initialization.
\vspace{-1mm}

\paragraph{Parameter-Free Update Strategy.}
In addition to $\sigma_0$ there are three more parameters in Algorithm \ref{alg:adaptive_cocoa} -- namely $\zeta$, $\gamma$ and $\xi$ -- that determine how to update $\sigmat$. The most natural choice for $\xi$ is a small positive value, as we do not want to discard updates that would yield a function decrease; we therefore choose $\xi=0$. The convergence of Algorithm~\ref{alg:adaptive_cocoa} is guaranteed for any choice of $\zeta,\gamma>1$, and we found empirically that the performance is not very sensitive to the choice of these parameters and the optimal values are robust across different datasets (e.g., $\gamma=\zeta \approx 1.2$ is generally a good choice). However, to completely eliminate these parameters from the algorithm we suggest the following practical parameter-free update schedule:
 \[\small{\sigma_{t+1}:= \frac {f(A(\alphat\!+\!\Dav))-f(A\alphat)-\nabla f(A\alphat)A\Dav}{\hat f(A\alphat, A\Dav)-f(A\alphat)-\nabla f(A\alphat )A\Dav}\sigmat}.\]
This scheme is not only parameter-free, but it also adapts~$\sigma$ proportionally to the misfit of the model. The evaluation of this scaling factor does not add any additional computation to the evaluation of $\rho_t$. Note that for this scheme to meet the required conditions of convergence presented in Section~\ref{sec:analysis}, we need to ensure that the sequence of $\sigma_t$ is bounded, which can easily be done by defining an arbitrary maximum value although we empirically found that this was not necessary.
Because of this appealing property of not requiring any tuning we will use this strategy for the following experiments.

\begin{figure}[!t]
	\vspace{-0.3cm}
\centering
\subfigure[K=32]{\label{fig:L1webspam}\includegraphics[width=0.495\columnwidth]{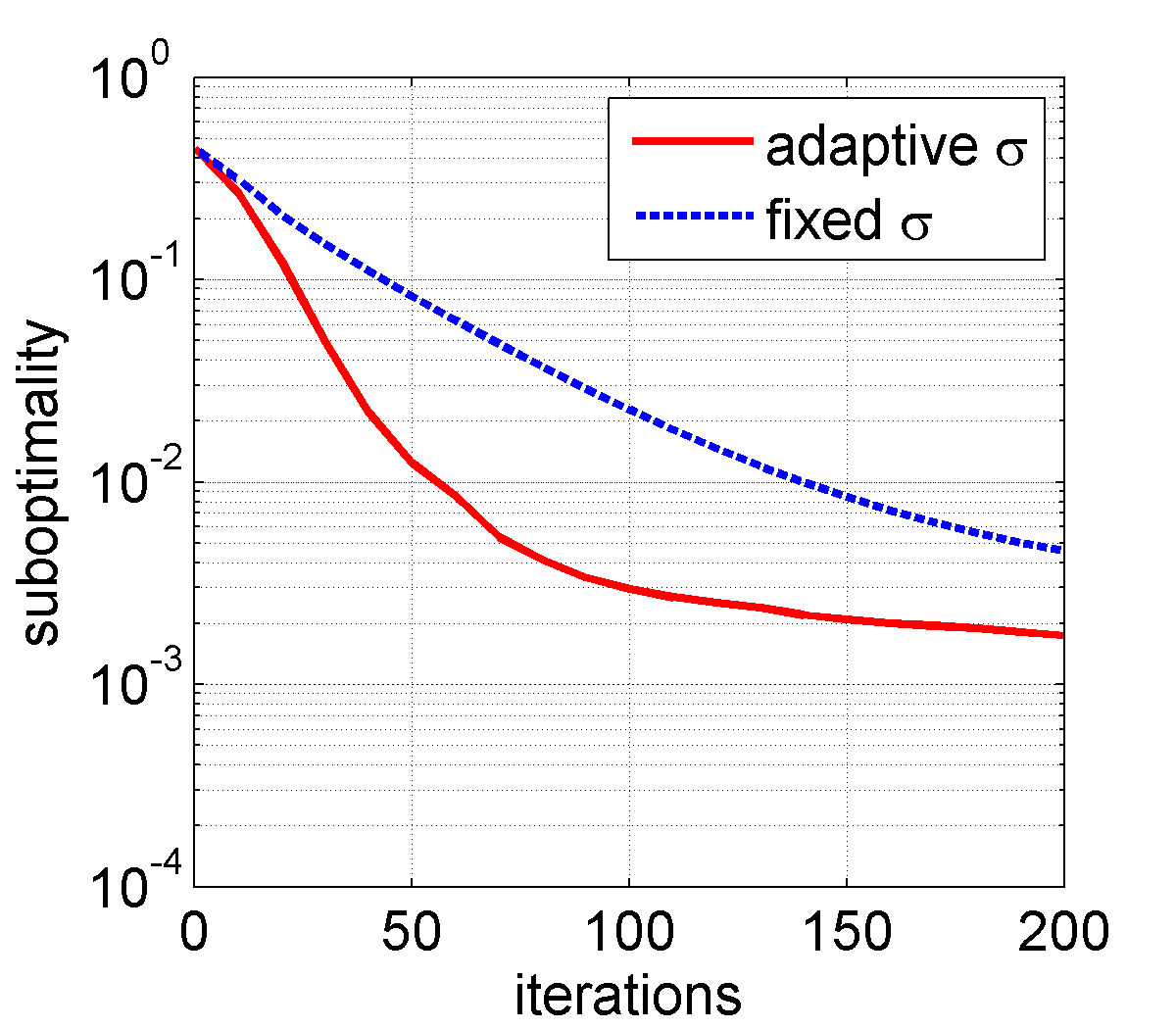}}
\subfigure[K=16]{\label{fig:L1webspam}\includegraphics[width=0.495\columnwidth]{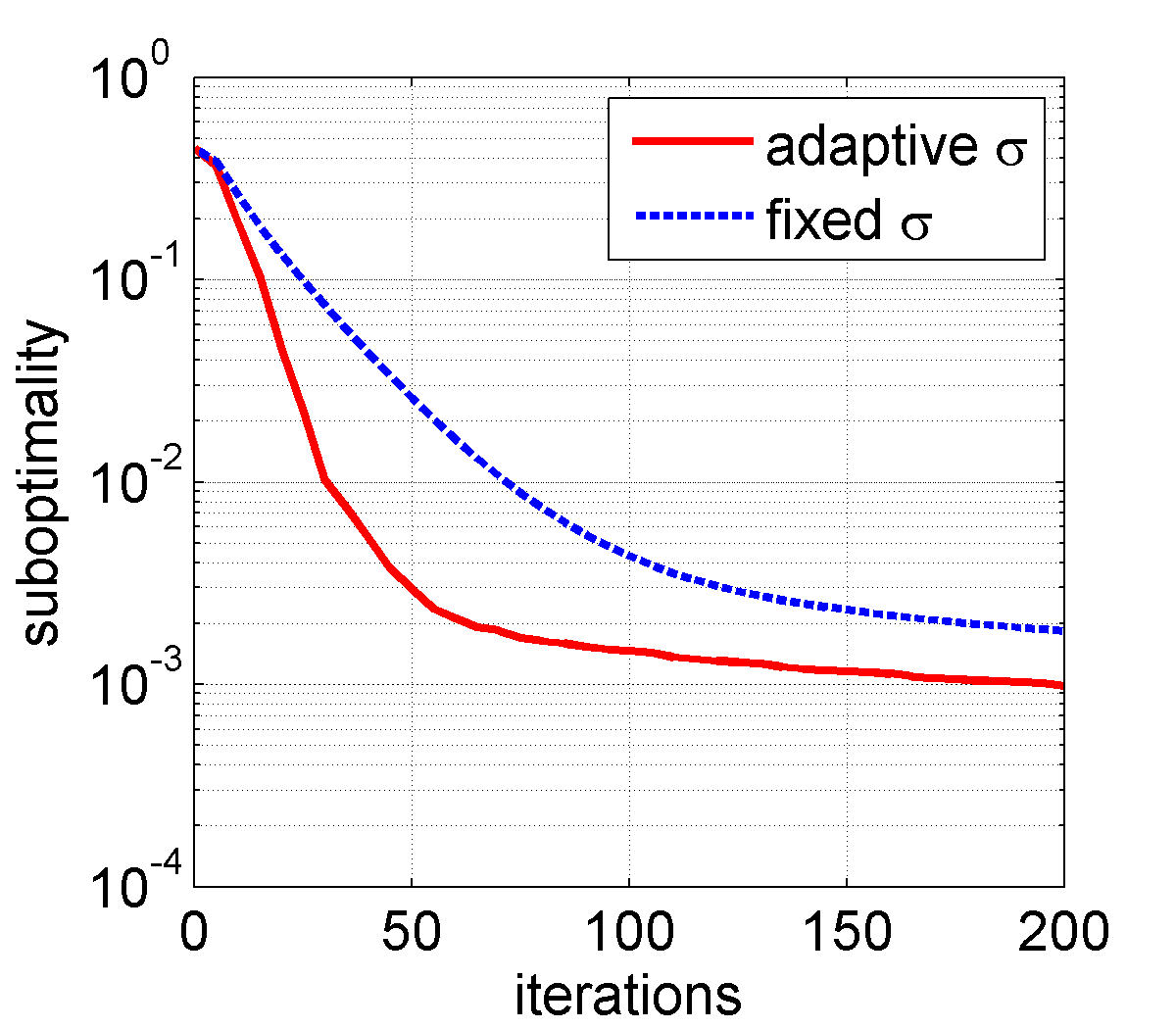}}
\vspace{-0.2cm}
\subfigure[K=8]{\label{fig:L1webspam}\includegraphics[width=0.495\columnwidth]{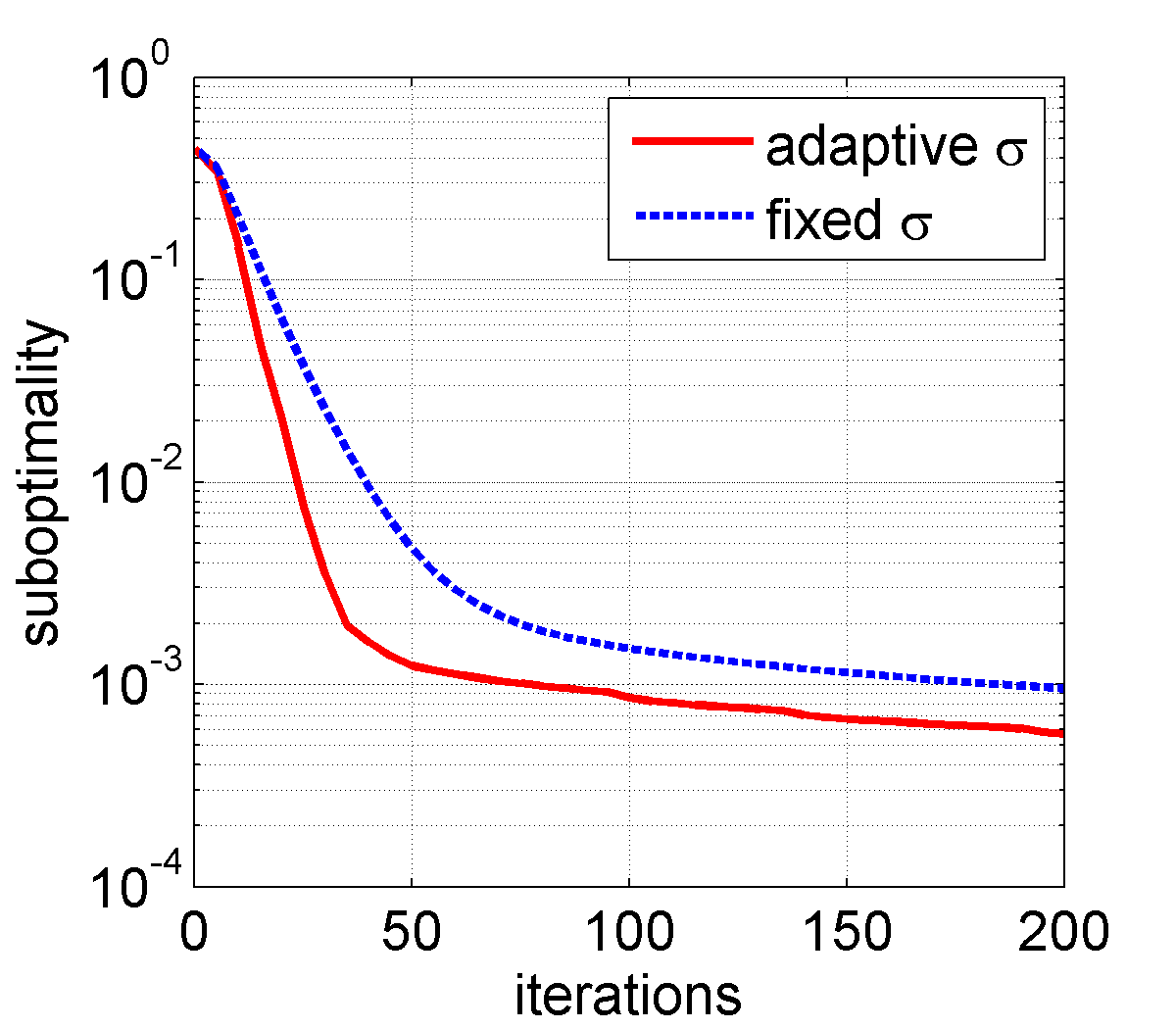}}
\subfigure[K=4]{\label{fig:L1webspam}\includegraphics[width=0.495\columnwidth]{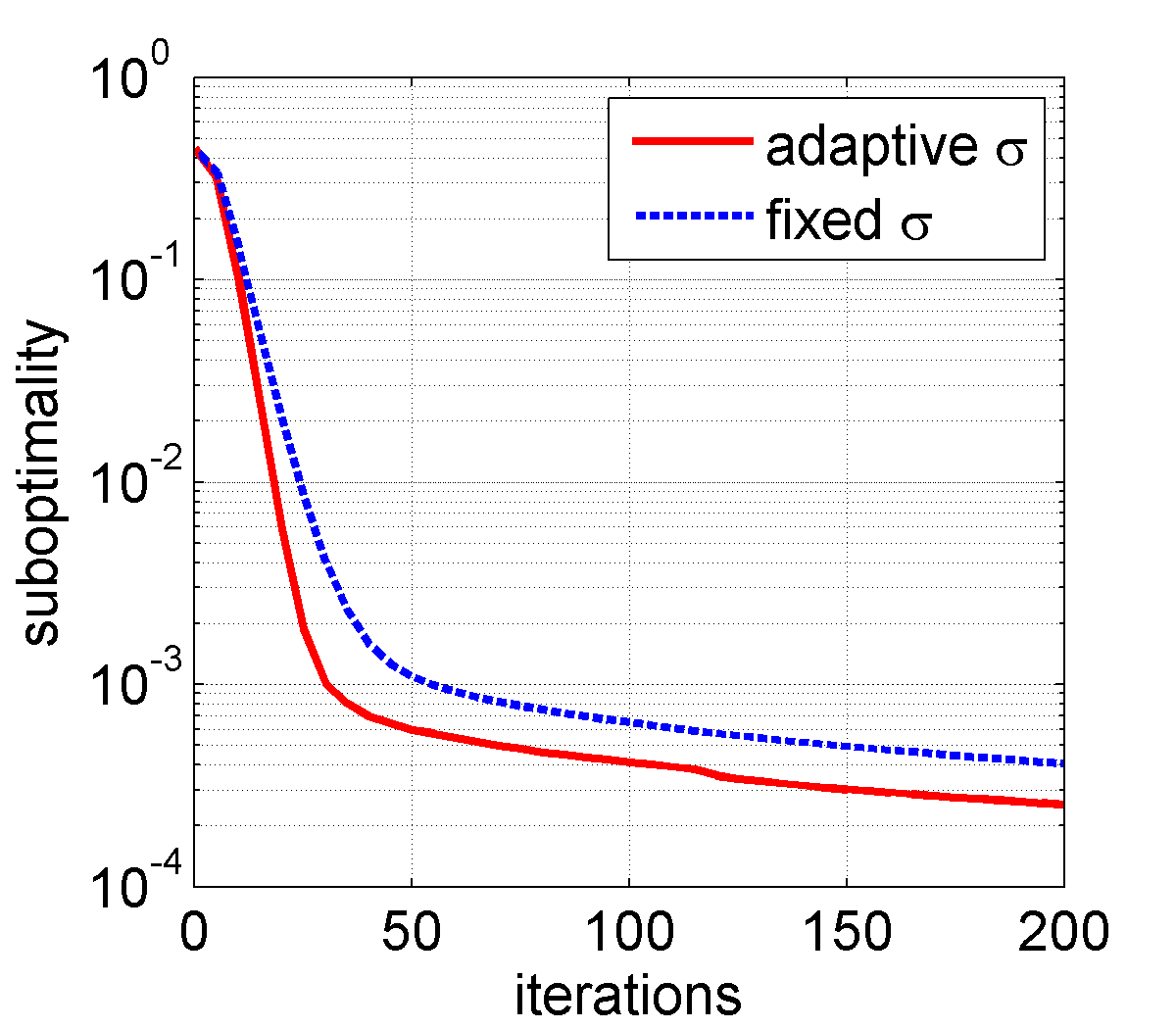}}
\caption{Comparison of using an adaptive approach for $\sigma$ vs. using a fixed safe value for $\sigma$ for different numbers of workers ($K$). Training $L_2$ logistic regression on a subsample (10 million examples) of the criteo dataset.\vspace{-0.2cm}}
\label{fig:N}
\end{figure}

\begin{figure*}[t!]
\centering
\subfigure[url]{\label{fig:L1url}\includegraphics[width=0.5\columnwidth]{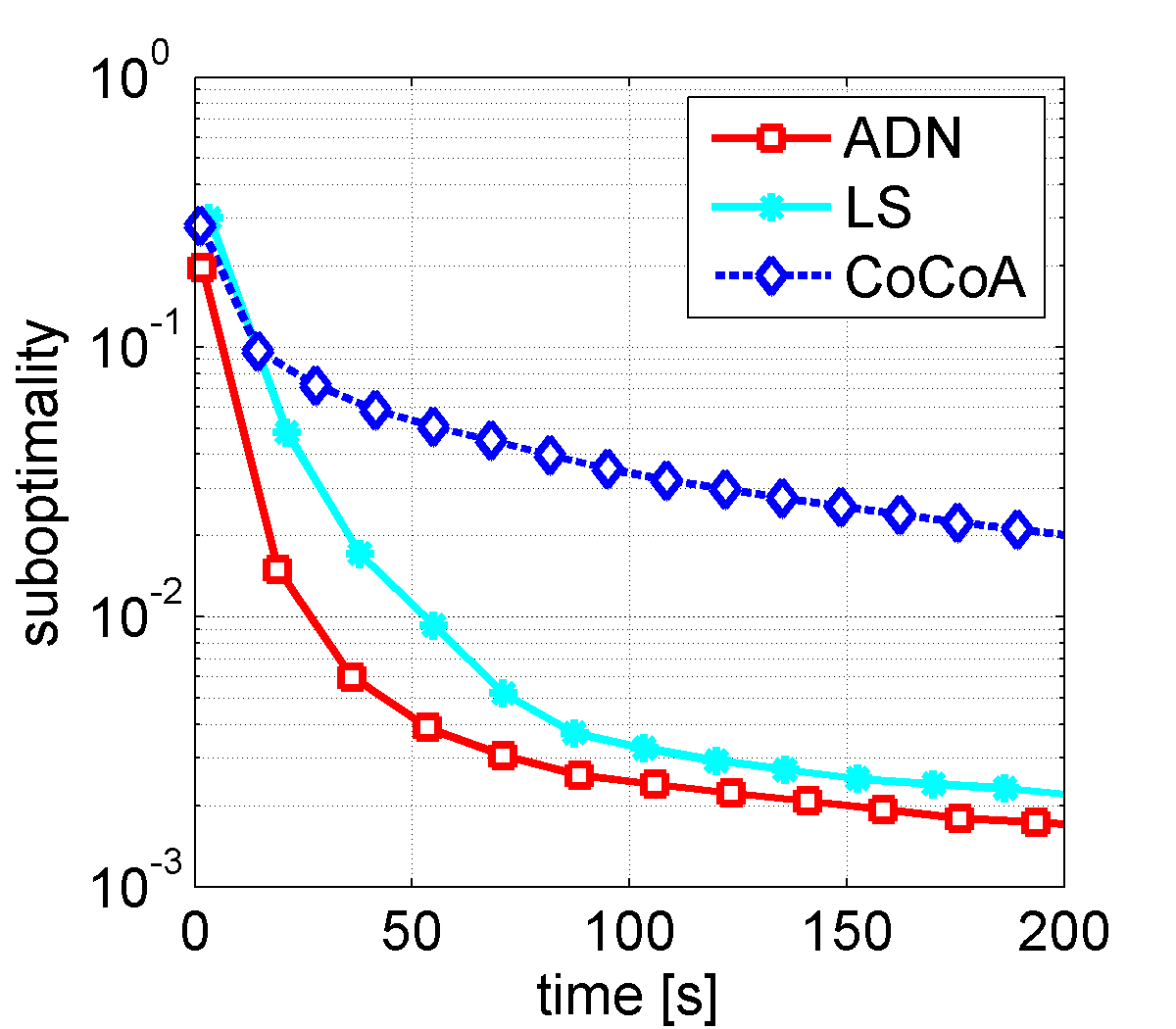}}
\subfigure[webspam]{\label{fig:L1webspam}\includegraphics[width=0.5\columnwidth]{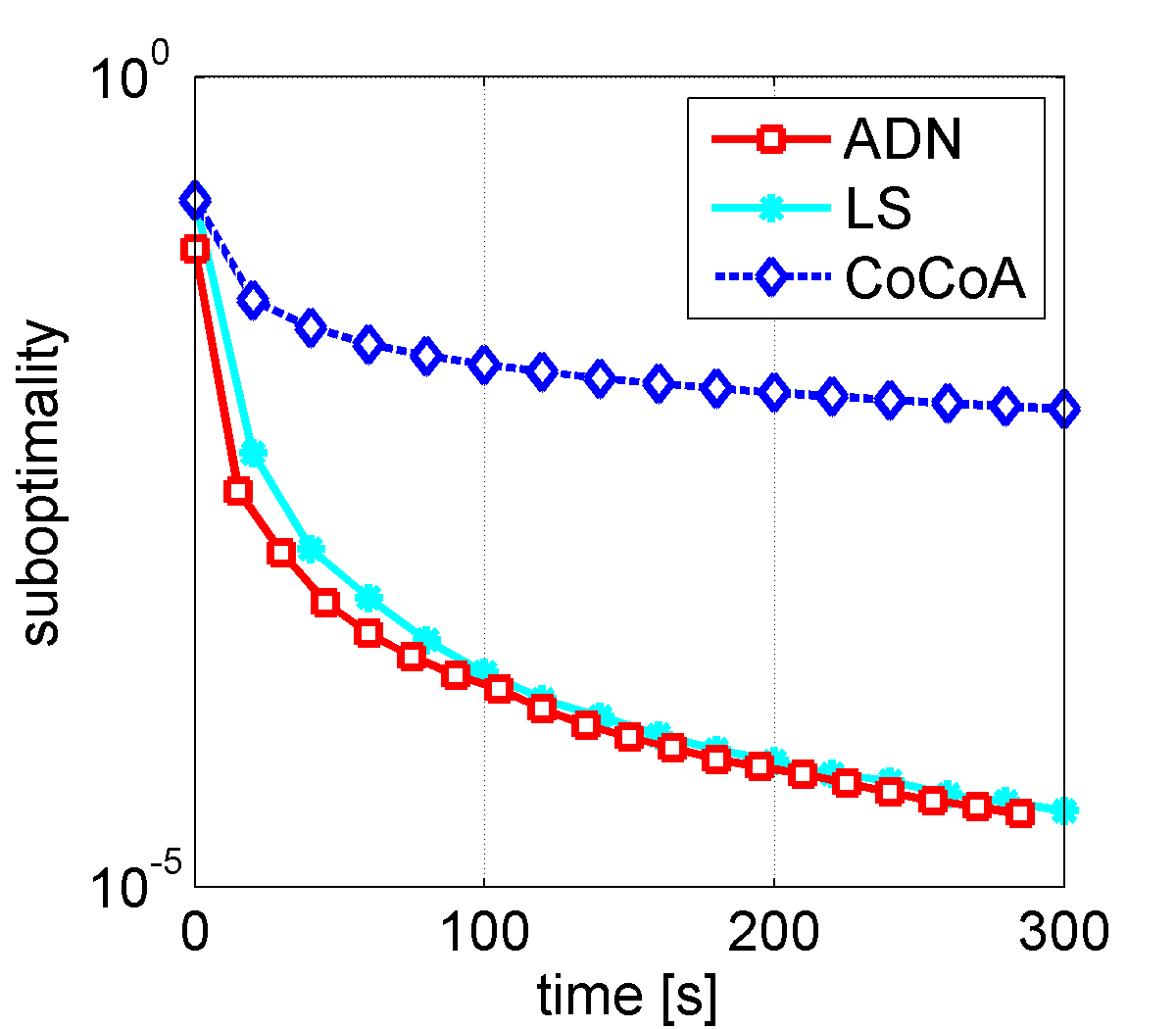}}
\subfigure[kdda]{\label{fig:L1kdda}\includegraphics[width=0.5\columnwidth]{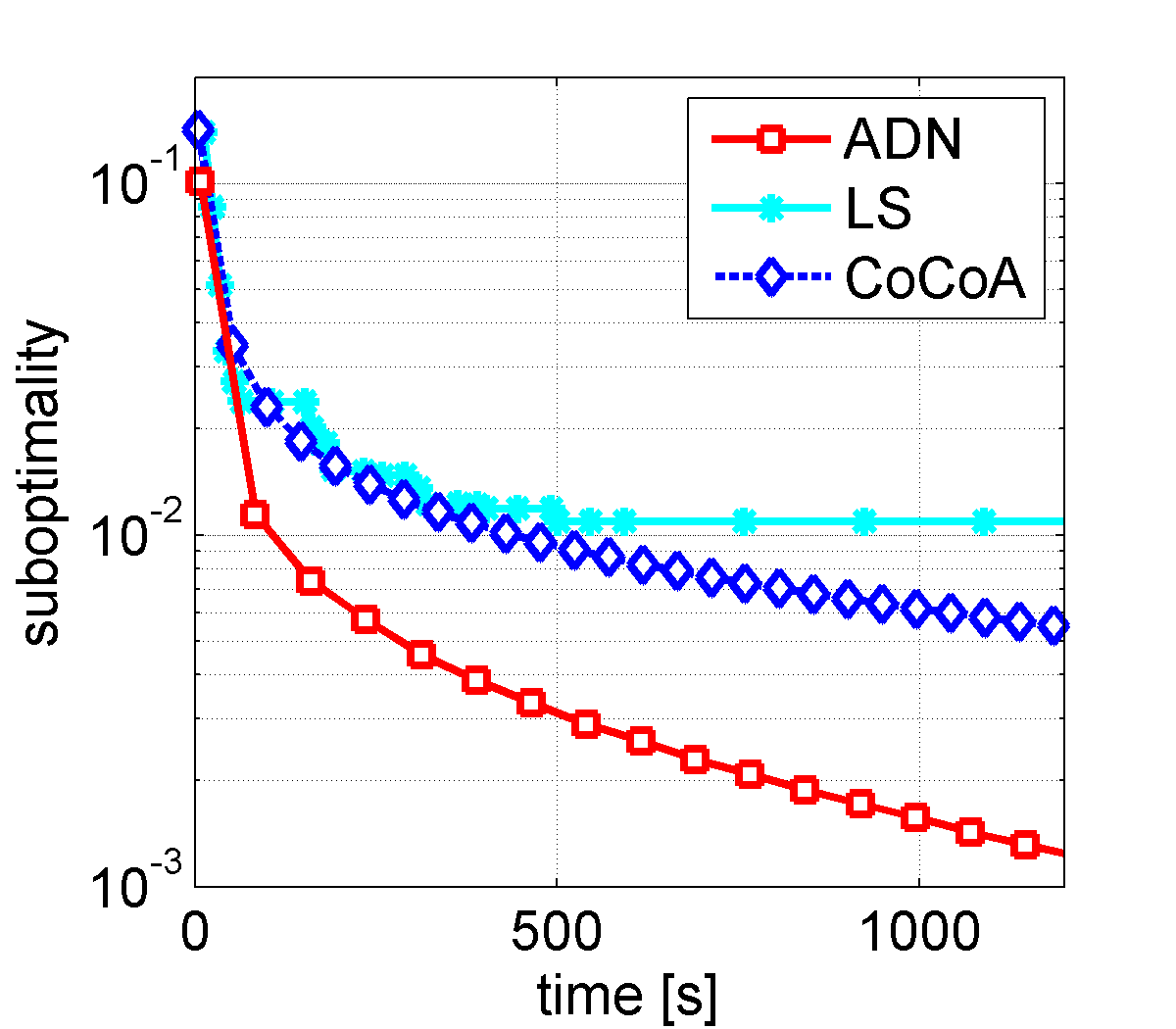}}
\subfigure[criteo]{\label{fig:L1criteo}\includegraphics[width=0.5\columnwidth]{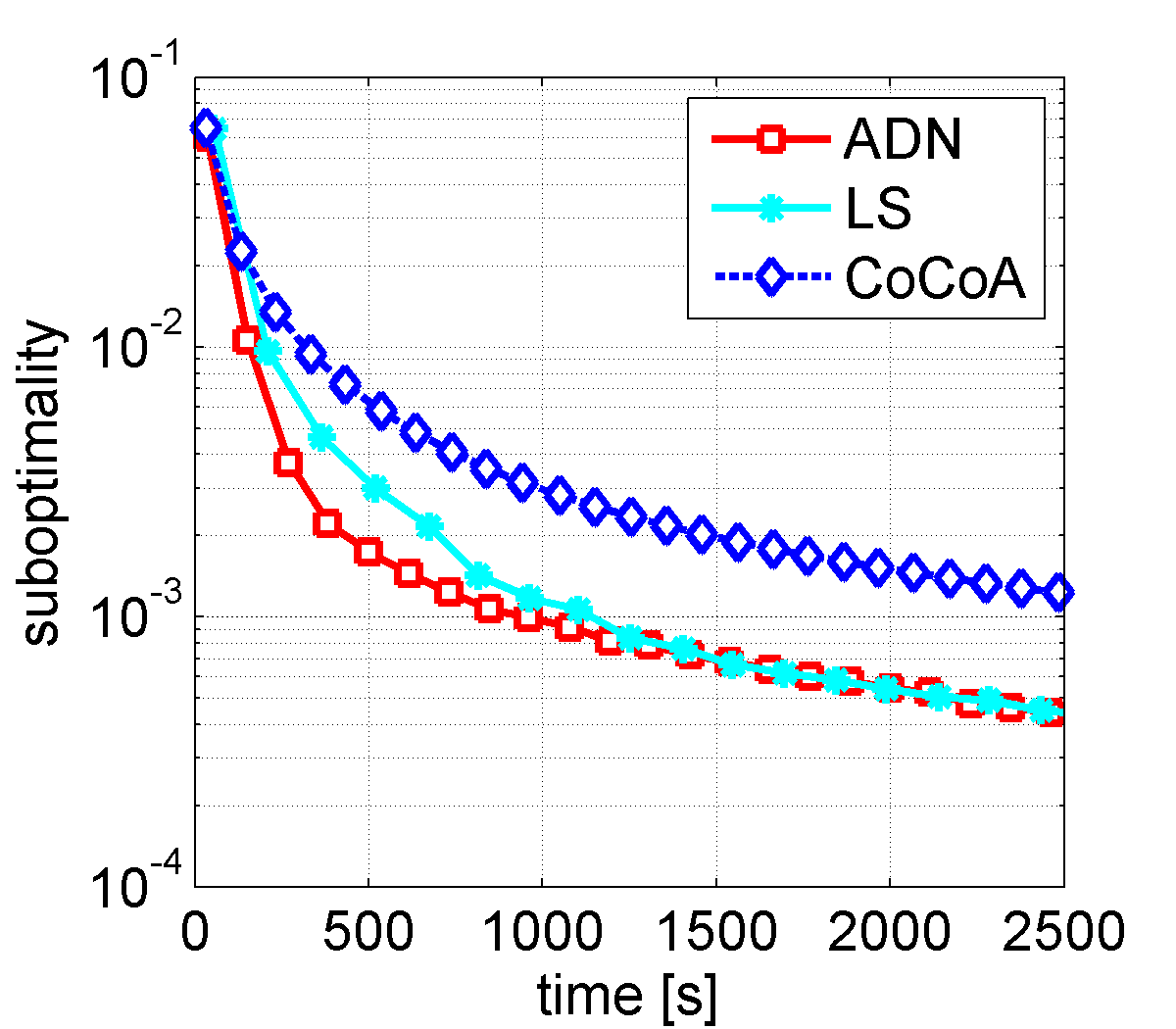}}
\vspace{-0.3cm}
\caption{Performance comparison of primal solver for $L_1$-regularized Logistic Regression.\vspace{-0.1cm}}
\centering
\label{fig:L1}
\end{figure*}

\vspace{-1mm}
\paragraph{Gain of Adaptive Strategy.}
In this section we investigate the benefits of using an adaptive $\sigma$ as opposed to a static one. We focus on a dual $L_2$-regularized logistic regression model where $f$ is a quadratic function and thus, its Hessian corresponds to a scaled identity matrix. This allows us to study the effect of adaptivity in isolation. It also allows us to compare to a reference model with $\sigma=K$ which comes with convergence guarantees, see~\cite{Smith:2016wp}. In Figure~\ref{fig:N} we compare the two approaches and observe that with an increasing number of workers, the gains provided by the adaptive approach increase. This comes from the fact that the more workers we have, the less accurate the block diagonal approximation 
in the auxiliary model is and thus it is increasingly difficult to establish a safe fixed value for $\sigma$ that covers any partitioning of the data in an ad hoc fashion.
Note that the adaptive strategy does not only improve over the safe fixed value of $\sigma$ as shown in Figure \ref{fig:N} but it also enables convergence for objectives to be guaranteed where no tight practical bound is known.

\subsection{Performance for Logistic Regression}
\label{subsec_exp_lr}
We now analyse the performance of ADN for training a Logistic Regression model on multiple large-scale datasets and compare it to different state-of-the-art methods. First, we will consider $L_2$ regularization, which results in a strongly-convex objective function. This enables the application of a broad range of existing methods. In the second part of this section we focus on $L_1$ regularization, where -- to the best of our knowledge -- the only existing baselines that come with convergence guarantees are CoCoA~\cite{Smith:2016wp} and slower mini-batch proximal SGD.
\vspace{-2mm}

\paragraph{Baselines.}
We compare our approach against \textit{GIANT} as a representative scheme for the class of approximate Newton methods. This approach was shown in \cite{wang2017giant} to achieve competitive performance to other similar algorithms such as DANE or DiSCO. The main difference between these methods and ours is that they build updates based on a local approximation of the full Hessian matrix, whereas we work with exact blocks of the full Hessian matrix. In order to establish a fair comparison, we re-implemented GIANT using MPI while following the open source implementation provided by the authors\footnote{\url{https://github.com/wangshusen/SparkGiant}}. We use conjugate gradient descent as a local solver and implemented the suggested backtracking line-search approach.

Our second baseline is the approach presented in~\cite{lee2017distributed} which is similar to ours as it builds on the same block diagonal approximation of the Hessian matrix. However, it uses a fixed model and then relies on a backtracking line search approach to guarantee convergence. We will refer to this scheme as \textit{LS} in our experiments.

The third baseline is \textit{CoCoA} which approximates the Hessian $\nabla^2 f(.)$ by a scaled identity matrix using the smoothness property of $f$. Their quadratic model performs well if $f$ is indeed a quadratic function such as the least squares loss or the dual of the $L_2$ regularizer. However, we will see that this is not a good model for the logistic loss function.

\begin{figure}[t!]
\centering
\subfigure{\label{fig:L2url}\includegraphics[width=0.49\columnwidth]{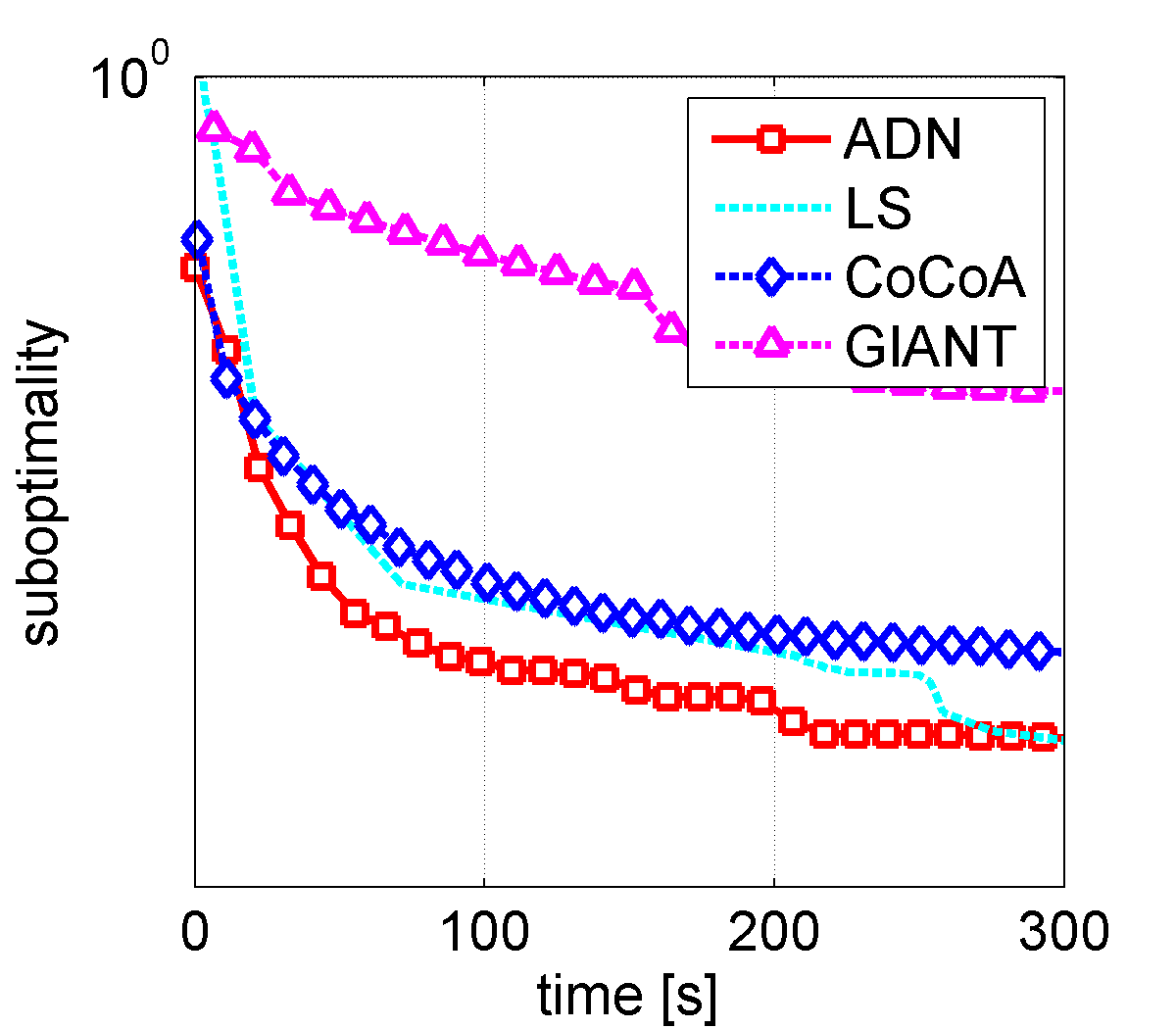}}
\subfigure{\label{fig:L2criteo}\includegraphics[width=0.49\columnwidth]{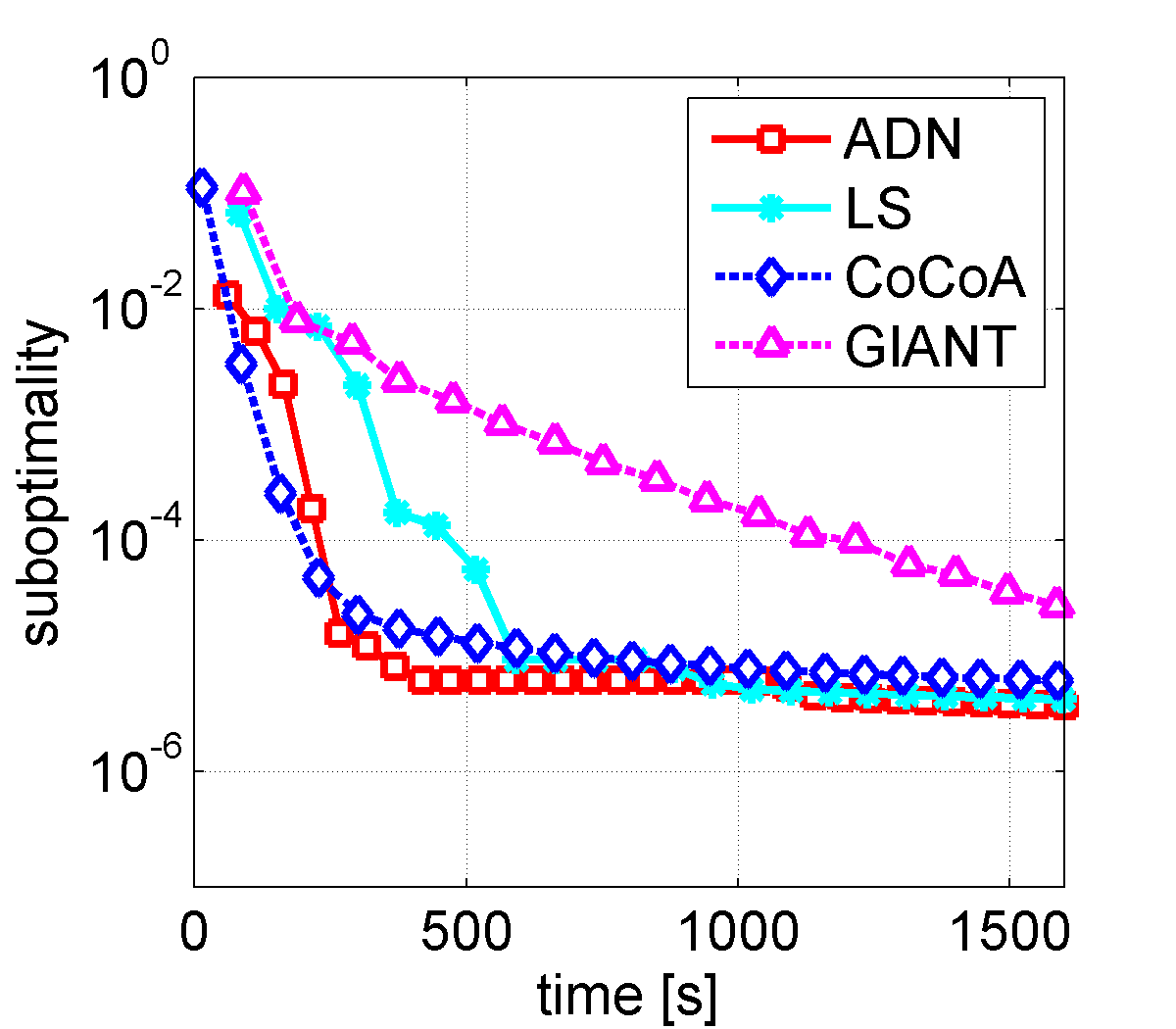}}\\[-2mm]
\subfigure{\label{fig:L2url}\includegraphics[width=0.49\columnwidth]{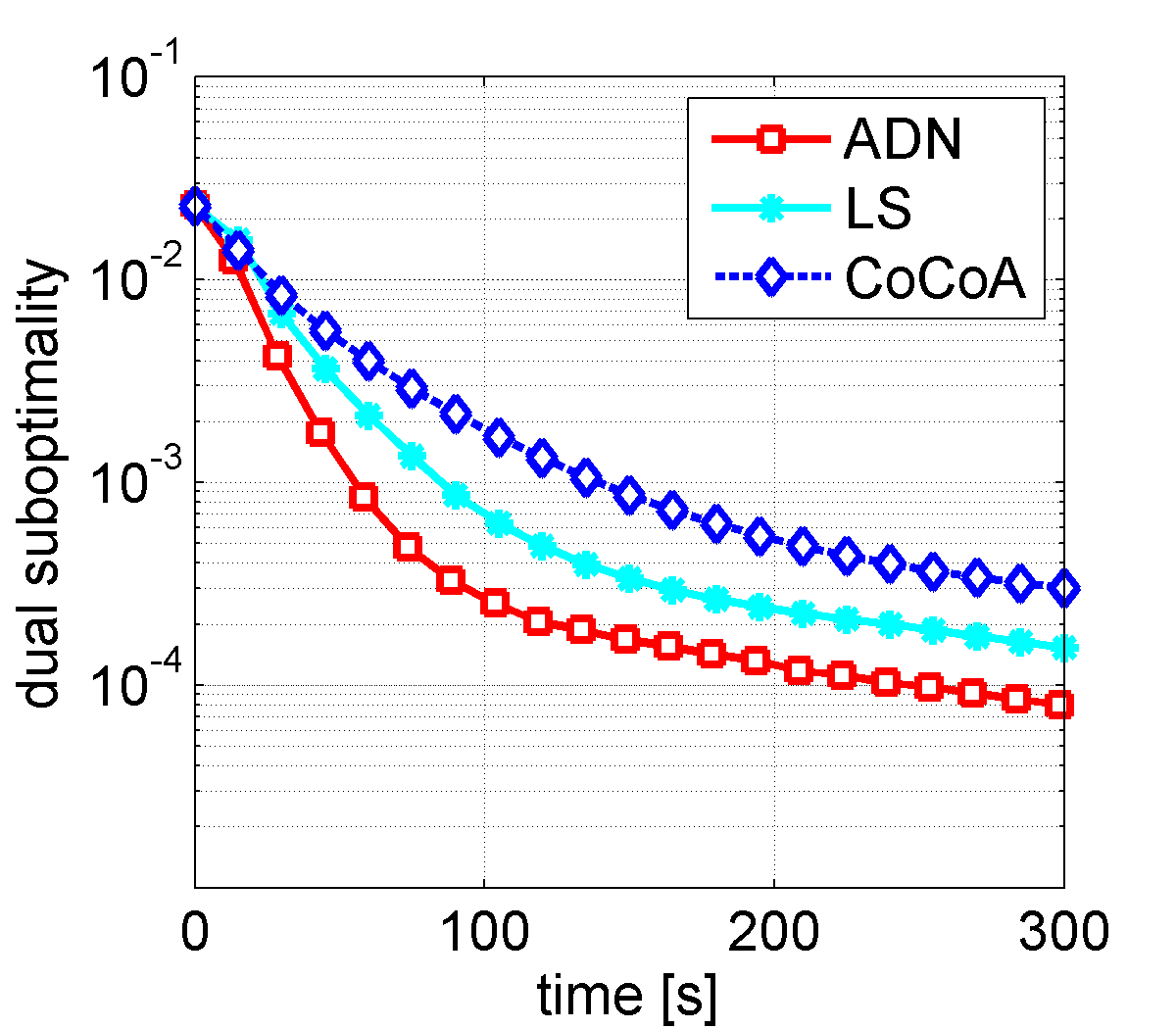}}
\subfigure{\label{fig:L2criteo}\includegraphics[width=0.49\columnwidth]{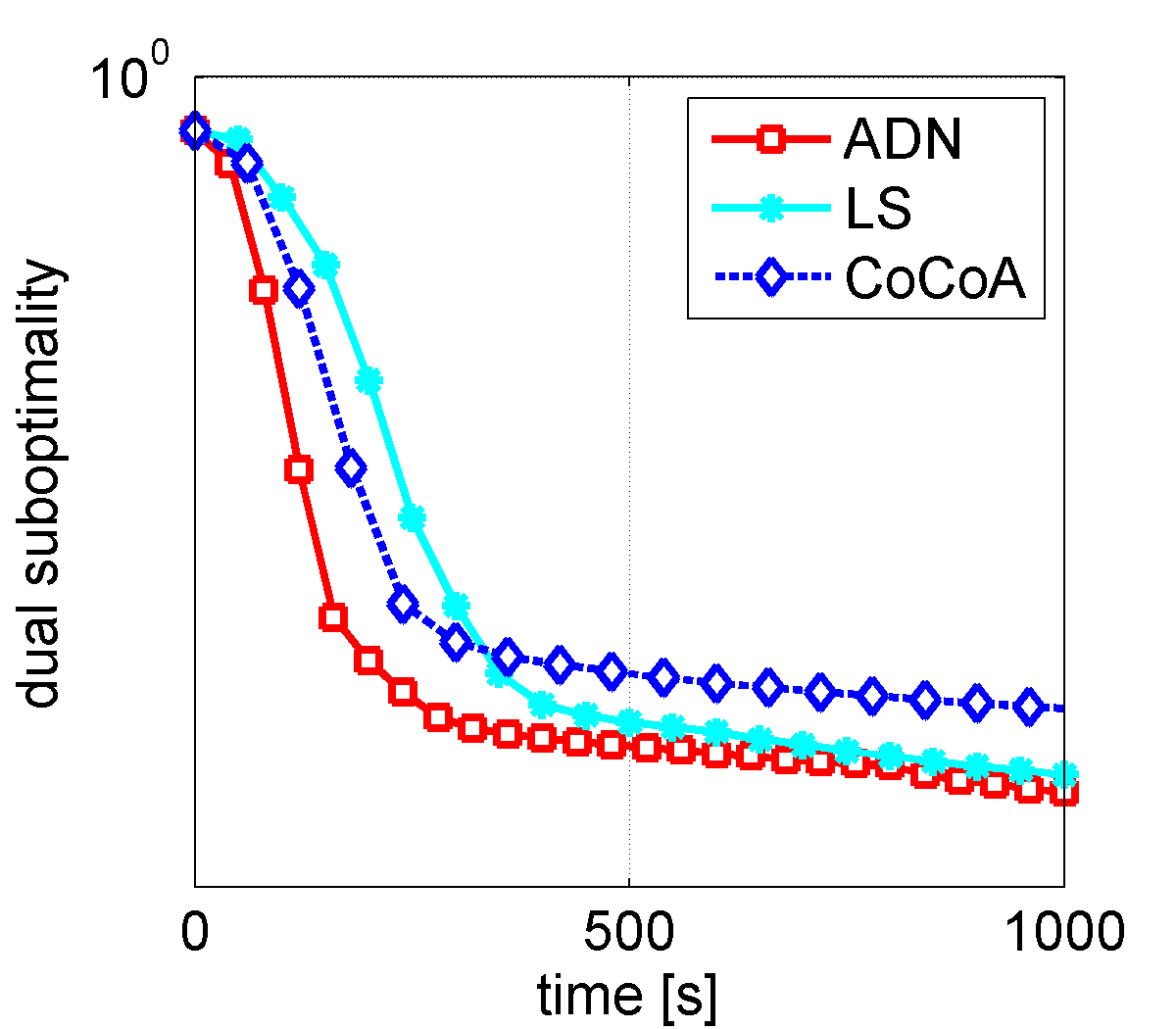}}
\vspace{-0.2cm}
\caption{Performance comparison for $L_2$-regularized logistic regression on url dataset (left) and criteo dataset (right) for solving the primal problem (top) and the dual problem (bottom).} 
\label{fig:L2}
\vspace{-0.5cm}
\end{figure}
\vspace{-1mm}

\paragraph{$L_1$ Regularization.} We consider the $L_1$-regularized logistic regression problem on the datasets introduced in Table~\ref{tbl:datasets}. We compare CoCoA (applied to the L1 primal problem) and {LS} to  {ADN} in Figure~\ref{fig:L1}. In general, we see significant gains from {ADN} over CoCoA which can be attributed to CoCoA using a quadratic approximation to the logistic function which is not a good fit. The performance of {LS} is similar or slightly worse than our approach, depending on the dataset. However, as shown in Figure~\ref{fig:L1kdda}, it can be unstable since the line-search approach used in~\cite{lee2017distributed} does not come with any theoretical guarantees for functions that are not strongly-convex. 
\vspace{-1mm}

\paragraph{$L_2$ Regularization.}
For $L_2$-regularized logistic regression, CoCoA, ADN and LS use a dual solver. The results presented in Figure \ref{fig:L2} show that CoCoA is competitive in this case since it uses the same block diagonal approximation of the Hessian matrix and benefits from cheap iterations as no function evaluations are needed. However, we can see that using an adaptive strategy nevertheless pays off and we can achieve a gain over CoCoA. For very high accuracy solutions (<$10^{-6}$), a solver that uses the full Hessian should be preferred if possible.

%% file: 06_conclusion.tex

\section{Conclusion}

We have presented a novel distributed second-order algorithm that optimizes an auxiliary model with a block-diagonal Hessian matrix. The separable structure of this model makes its optimization easily parallelizable. Each worker optimizes its own local model and sends a minimal amount of information to the master node. Our framework therefore avoids the computation and communication of an expensive Hessian matrix. In order to adjust for the approximation error of the model, we proposed using an adaptive scheme that resembles trust-region methods. This allows us to derive global guarantees of convergence for convex functions. Specializing our approach to strongly-convex functions recovers convergence results derived by existing distributed second-order methods. From the practical side, we have proposed a parameter-free version of our algorithm, discussed how to develop an efficient implementation and demonstrated significant speed-ups over state-of-the-art baselines on several large-scale datasets.

%% file: appendix.tex

\appendix
\setlength{\belowdisplayskip}{5pt} \setlength{\belowdisplayshortskip}{3pt}
\setlength{\abovedisplayskip}{5pt} \setlength{\abovedisplayshortskip}{3pt}
\part*{Appendix}

\section{Analysis}
In order to prove convergence of Algorithm \ref{alg:adaptive_cocoa} we proceed as follows: 
\begin{enumerate}
\item For the auxilary model with block diagonal $\tilde H$ as described in Section~\ref{sec:block_model}, we lower bound the decrease in the model achieved in each step of Algorithm \ref{alg:adaptive_cocoa}. This yields the lower bound on $\model_\sigmat(\0;\alphav)-\model_\sigmat(\Dav;\alphav)$ provided in Lemma~\ref{lem:dec}. 
\item For iterations that are successful (i.e.,  they achieve $\rho_t\geq\xi$ and thus successfully decrease the objective function, see Definition \ref{def_successful_update}), the construction of  Algorithm \ref{alg:adaptive_cocoa} allows us to relate the model decrease to the function decrease $\bO(\alphav)-\bO(\alphav+\Dav)$ through the constant $\xi$. This lets us establish a convergence rate in terms of the number of successful iterations, which is shown in Lemma \ref{lem:cocoac}, \ref{lem:cocoasc} for non-strongly convex $g_i$ and strongly convex $g_i$, respectively.

\item Finally, in order to establish overall convergence of Algorithm \ref{alg:adaptive_cocoa} we need to bound the number of  unsuccessful iterations (i.e.,  iterations for which $\rho_t < \xi$ where no update is applied to the model parameters). This is accomplished by showing that the construction of the sequence $\{\sigma_t\}_{t\geq 0}$ is such that the number of successive unsuccessful iterations is limited and Algorithm \ref{alg:adaptive_cocoa} will therefore eventually yields a successful iteration which will allow us to decrease the objective function. In details, this is accomplished as follows:
\begin{itemize}
\item show that Algorithm \ref{alg:adaptive_cocoa} finds a successful step as soon as the penalty parameter $\sigmat$ exceeds some critical value, thereby  the sequence $\{\sigma_t\}_{t\geq 0}$ is guaranteed to stay within some bounded positive interval.
\item use the boundedness of $\sigmat$ to establish an upper bound on the maximum number of unsuccessful iterations  and hence the total number of steps to reach a target suboptimality.
\item lastly in Section \ref{app_upperbound} we establish the boundedness of the sequence $\{\sigma_t\}_{t\geq 0}$ for two general situations. 
\end{itemize}
\end{enumerate}
\subsection{Model Decrease}


\lemmathree*

\begin{proof}
Given that the updates $\Dak$ optimize the respective local models (defined in \eqref{eq:theta_approx}) $\eta$-approximately, we can relate the model decrease provided by $\Dav = \sum_k \Dak$ to the optimal model decrease as follows:
\begin{eqnarray*}\model_\sigmat( \0;\alphat) - \model_\sigmat(\Dav;\alphat)&=&\model_\sigmat( \0;\alphat) - \sum_k \mstk(\Dak;\alphat)
\\&\geq&\model_\sigmat(\0;\alphat) -  \sum_k[(1-\eta) \mstk({\Dav}^\star_{[k]};\alphat)+\eta \mstk(\0;\alphat)]\\
&=&\model_\sigmat(\0;\alphat) -  [(1-\eta) \mst(\Dav^\star;\alphat)+\eta \mst(\0;\alphat)]\\[0.2cm]
&=& (1-\eta)[\model_\sigmat(\0;\alphat)- \model_{\sigmat}(\Dav^\star;\alphat)],
\end{eqnarray*}
where ${\Dav}^\star_{[k]} =\argmin_{\Dak}\mstk({\Dav}^\star_{[k]};\alphat) $ and  $\Dav^\star =\sum_k{\Dav}^\star_{[k]}$.

From here we proceed by bounding the model decrease for the optimal update, i.e., $\Delta_\model:=\model_\sigmat(\0;\alphat) - \model_\sigmat(\Dav^\star;\alphat)$ which, using \eqref{eq:modelCoCoA}, can be written as
\begin{eqnarray*}
\Delta_\model &=& - \nabla f(A\alphav)^\top A\Dav^\star -  \frac \sigma 2\sum_k {\Delta\alphav^\star_{[k]}}^\top \tilde H(\alphav)\Delta\alphav^\star_{[k]}+\sum_i g_i(\alpha_i) - \sum_i g_i((\alphav+\Dav^\star)_i).
\end{eqnarray*}
 where we omit the superscript $t$ for reasons of readability.
Since $\Dav^\star$ is the minimizer of $\model_\sigmat(\Dav; \alphav)$ the following inequality must hold for an arbitrary update direction $\tilde \sv$:
\begin{eqnarray}
	\Delta_\model &\geq& - \nabla f(A\alphav)^\top A\tilde \sv - \frac \sigma 2\sum_k {\tilde \sv_{[k]}}^\top\tilde H(\alphav)\tilde \sv_{[k]}+ \sum_i g_i (\alpha_i) - g_i ((\alphav+\tilde \sv)_i). 
	 \label{eq:arbitrarystep}
\end{eqnarray}
Hence, let us consider the specific update $\tilde \sv = \kappa(\uv-\alphav)$ for some $\kappa\in(0,1]$ and $\uv\in \R^n$. We find
\begin{eqnarray}
\Delta_\model &\geq&- \nabla f(A\alphav)^\top A\kappa(\uv-\alphav) -\frac {\kappa^2 \sigma} 2\sum_k (\uv-\alphav)^\top_{[k]}\tilde H(\alphav)(\uv-\alphav)_{[k]}\\&& + \sum_i g_i(\alpha_i) - \sum_i g_i((\alphav+\kappa(\uv-\alphav))_i).
\label{eq:specificstep}
\end{eqnarray}
Furthermore, using $\mu$-strong convexity of $g_i$ with $\mu\geq 0$, (i.e., the bound also holds for $\mu = 0$ in which case $g_i$ is convex), we get
\[g_i((1-\kappa)\alpha_i+\kappa u_i)  \leq (1-\kappa) g_i(\alpha_i)+ \kappa g_i(u_i) - \frac \mu 2 \kappa (1-\kappa)(\alpha_i-u_i)^2,\]
which combined with~\eqref{eq:specificstep} yields
\begin{eqnarray}
\Delta_\model &\geq&\kappa \sum_i g_i(\alpha_i) - \kappa \sum_i g_i(u_i) + \sum_i \frac \mu 2 \kappa (1-\kappa)(\alpha_i-u_i)^2\notag\\&& - \nabla f(A\alphav)^\top A\kappa(\uv-\alphav)-\frac {\kappa^2 \sigma} 2\sum_k (\uv-\alphav)^\top_{[k]}\tilde H(\alphav)(\uv-\alphav)_{[k]}\notag\\
&=&\kappa  \underbrace { \left[\sum_i g_i(\alpha_i) - \sum_i g_i(u_i) - \nabla f(A\alphav)^\top A (\uv-\alphav) \right] }_{\text{(gap)}}   + \frac \mu 2 \kappa (1-\kappa)\|\alphav-\uv\|_2^2 \notag \\&&-\frac {\kappa^2 \sigma} 2\sum_k (\uv-\alphav)^\top_{[k]}\tilde H(\alphav)(\uv-\alphav)_{[k]}.\label{eq:bound1}
\end{eqnarray}
To further simplify this bound we choose $\uv$ such that $u_i\in\partial g_i^*(-\xv_i^\top \wv (\alphav))$ where $\xv_i$ denote the columns of the data matrix $A$, $\wv(\alphav) := \nabla f(A\alphav)$ and $g_i^*$ denotes the convex conjugate of the function $g_i$.  
For this particular choice the  term ``$\text{(gap)}$'' in \eqref{eq:bound1} corresponds to the duality gap of the objective at the iterate $\alphav$. To see this, note that the duality gap (see, e.g., \cite{duenner16}) for \eqref{eq:A} can be written as
\begin{eqnarray}
\Gap(\alphav) &=& \sum_i g_i^*(-\xv_i^\top \wv(\alphav)) + g_i(\alpha_i) + \alpha_i \xv_i^\top \wv(\alphav)\notag\\
&\overset{(a)}{=}& \sum_i u_i(-\xv_i^\top\wv(\alphav)) - g_i(u_i) + g_i(\alpha_i) + \alpha_i \xv_i^\top \wv(\alphav)\notag\\
&=& \sum_i g_i(\alpha_i) - g_i(u_i) - (u_i-\alpha_i) \xv_i^\top \wv(\alphav),  \label{eq:gap}
\end{eqnarray}
where equality $(a)$ holds for any $u_i\in\partial g_i^*(-\xv_i^\top \wv)$ since for such an optimal $u_i$ the  Fenchel-Young inequality holds with equality, i.e.,
\begin{equation}g_i(u_i) = u_i(-\xv_i^\top\wv) - g_i^*(-\xv_i^\top \wv)\label{eq:fenchel}.
\end{equation}
Now combining \eqref{eq:bound1} with \eqref{eq:gap} and \eqref{eq:fenchel} we find
\begin{eqnarray}
\Delta_\model &\geq&\kappa \Gap(\alphav) + \frac \mu 2 \kappa (1-\kappa)\|\alphav-\uv\|_2^2-\frac {\kappa^2 \sigma} 2\sum_k (\uv-\alphav)^\top_{[k]}\tilde H(\alphav)(\uv-\alphav)_{[k]}\label{eq:boundLem}
\end{eqnarray}
and Lemma \ref{lem:dec} follows.
\end{proof}


\subsection{Function Decrease}
In order to relate the model decrease to the function decrease we use the fact that every update $\Dav$ applied to the parameter vector in Algorithm \ref{alg:adaptive_cocoa} is successful in the following sense.

\begin{definition}[successful update]
	\label{def_successful_update}
	 The update $(\Dav, \sigmat)$ is called successful if the following inequality is satisfied:
\begin{equation}
\xi\leq\rho_t :=\frac {\bO(\alphat) - \bO(\alphat+\Dav)}{\bO(\alphat)  - \model_\sigmat(\Dav;\alphat)}
\label{eq:zeta}
\end{equation}
otherwise it is called unsuccessful.
\label{def:goodUpdate}
\end{definition}

\lemmafour* 

\begin{proof}
Starting from~\eqref{eq:zeta}, and observing that $\bO(\alphat) = \model_\sigmat(\0; \alphat)$ we have
\begin{equation}
\bO(\alphat) - \bO(\alphat+\Dav) \geq \xi (\model_\sigmat(\0; \alphat) - \model_\sigmat(\Dav; \alphat)).
\end{equation}
Combining this inequality with Lemma~\ref{lem:dec} concludes the proof.
\end{proof}


\subsection{Rate of Convergence}
\label{sec:convergence}
Let $S$ denote the set of successful iterations as
\[S:=\{t\geq 0 : \text{ iteration $t$ is successful  in the sense of Definition \ref{def_successful_update}}\}\]
Further, let us define two disjoint index sets $\mathcal{U}_T$ and $\mathcal{S}_T$, which represent the un- and successful steps that have occurred up to some iteration $T>0$;
\[S_T:=\{t\leq T : t\in S\} \]
\[ U_T:=\{t\leq T: t \notin S\}.\]

Now, we will use Lemma~\ref{lem:f_dec} to establish convergence of Algorithm~\ref{alg:adaptive_cocoa} as a function of the number of successful iterations. Therefore, we will start with convex functions $g_i$ where we show sublinear convergence and then show that for strongly convex functions $g_i$, this result can be improved to obtain a linear rate of convergence.

\subsubsection{Non-strongly Convex  $g_i$.}

\begin{lemma}(non-strongly convex $g_i$) Let $f$ be $\tfrac 1 \tau$-smooth and  $g_i$ be convex with $L$-bounded support.
Assume the sequence $\{\sigmat\}_{t\geq 0}$ is bounded above by $\sigma_\text{sup}$.
 Then,  we can bound the suboptimality of Algorithm \ref{alg:adaptive_cocoa} as
\begin{eqnarray*}
\bO(\alphav^{(T)})-\bO(\alphav^\star) \leq \frac{2 (C_5\sigma_{\text{sup}}+ \varepsilon_0) }{\xi (1-\eta)}\frac 1 {|S_T|} 
\end{eqnarray*}
 where $|S_T|$ counts the number of successful updates up to iteration $T$, $\varepsilon_0:=\bO(\alphav_0)-\bO(\alphav^\star)$ and $C_5=\tfrac 4 \tau R^2 L^2$  with $\|A_{:,i}\|\leq R \;\;\forall i$.
 \label{lem:cocoac}
\end{lemma}

\begin{proof}
For non-strongly convex $g_i$ (i.e., $\mu=0$) we know from Lemma \ref{lem:f_dec} that for any  for successful update $(\Dav,\sigmat)$ the function decrease at iteration $t$ can be lower bounded as 
\begin{equation}
\bO(\alphat) - \bO(\alphat+\Dav) \geq \xi (1-\eta) \left[\kappa\Gap(\alphav^{(t)})  - \frac {\kappa^2  \sigmat} 2   (\uv^{(t)}-\alphav^{(t)})^\top  \tilde H(\alphav^{(t)}) (\uv^{(t)}-\alphav^{(t)})\right].
\label{eq:f_dec_L}
\end{equation}
For our block diagonal hessian approximation  
\[\tilde H(\alphav)=\diag(A_{:\cI_1}^\top \nabla^2f(A\alphav)A_{:\cI_1},\dots A_{:\cI_K}^\top \nabla^2f(A\alphav)A_{:\cI_K})\] 
it holds that
\begin{eqnarray}
 (\uv^{(t)}-\alphav^{(t)})^\top  \tilde H(\alphav^{(t)}) (\uv^{(t)}-\alphav^{(t)})\leq\frac 1 \tau \|A (\uv^{(t)}-\alphav^{(t)})\|^2 \leq\frac 4 \tau R^2  L^2\label{eq:boundR}
\end{eqnarray}
with $\|A_{:,i}\|\leq R \;\;\forall i$. Inequality \eqref{eq:boundR} relies on the assumption that $g_i$ has $L$-bounded support: a) by duality between Lipschitzness and Bounded-Support \cite{duenner16} of the univariate functions $g_i$ we have $|\alpha_i|\leq L$ since $\alpha_i$ is in the support of $g_i$ and b) by the equivalence between Lipschitzness and bounded subgradient we also have $|u_i|\leq L$ since $u_i\in\partial g_i^*(-\xv_i^\top \wv (\alphav))$. Together this yields $|u_i-\alpha_i|\leq 4 L^2$ and the bound \eqref{eq:boundR} follows.

In the following we assume that the sequence $\{\sigma_t\}_{t\geq0}$ is bounded by $\sigmasup$. We write $\varepsilon^{(t)}:=\bO(\alphav^{(t)})-\bO(\alphav^\star)$ for the suboptimality at step $t$ and use that the duality gap upper-bounds the suboptimality, i.e, $\Gap(\alphat)\geq \varepsilon^{(t)}$. Combining this with  \eqref{eq:f_dec_L} yields
\begin{eqnarray}
\varepsilon^{(t+1)}\leq(1-\kappa \xi(1-\eta))\varepsilon^{(t)} + \xi(1-\eta)\frac {2 \kappa^2 }  \tau R^2  L^2 \sigma_\text{sup}.
\label{eq:rec}
\end{eqnarray}

Let $a,b$ being positive constants defined as $a =\xi (1-\eta) $ and $b=\tfrac 4 {\tau} R^2  L^2 \sigma_\text{sup}$, then the above inequality can be written as
\begin{eqnarray}
\varepsilon^{(t+1)}\leq(1-\kappa a)\varepsilon^{(t)} +\frac { \kappa^2 } {2} a b.
\label{eq:rec2}
\end{eqnarray}
and holds for any $\kappa\in(0,1]$. Now let us choose $\kappa$ to minimize the RHS of \eqref{eq:rec2} which yields
$\kappa = \frac{1}{b}\varepsilon^{(t)}$ and to have $\kappa\in(0,1]$ we further constrain $\kappa$ to
\[\kappa = \min\left\{1 ,\frac{1}{b}\varepsilon^{(t)}\right\}.\]
Now let us consider the two cases separately:
\begin{enumerate}
\item \underline{$\varepsilon^{(t)}\geq  b $.} In this case we choose $\kappa =1$. Thus, from \eqref{eq:rec2} we get 
\begin{eqnarray}
\varepsilon^{(t+1)}\leq(1- a)\varepsilon^{(t)} +\frac { a b } {2} \leq(1- a)\varepsilon^{(t)} +\frac { a  } {2}\varepsilon^{(t)} = \left(1- \frac a 2\right)\varepsilon^{(t)} 
\label{eq:b1}
\end{eqnarray}
\item \underline{$\varepsilon^{(t)}<b$.} In this case we choose $\kappa = \frac 1 b \varepsilon^{(t)}$ and hence from \eqref{eq:rec2} we get 
\begin{eqnarray}
\varepsilon^{(t+1)}\leq \varepsilon^{(t)}- \frac a {2 b} \varepsilon^{(t)} \varepsilon^{(t)}.
\label{eq:b2}
\end{eqnarray}
\end{enumerate}
Note that the inequalities \eqref{eq:b1} and\eqref{eq:b2} together with non-negativity of $\varepsilon^{(t)}$ and $a\in (0, 1)$ imply that $\varepsilon^{(t+1)}\leq\varepsilon^{(t)}$ and thus $\{\varepsilon^{(t)}\}$ is a decreasing sequence.
Combining the two inequalities \eqref{eq:b1} and \eqref{eq:b2}  we get the following bound which holds for both cases and thus for every $t$:
\begin{eqnarray}
\varepsilon^{(t+1)}\leq \varepsilon^{(t)}-\frac a 2 \min\left\{ \frac 1 { b} \varepsilon^{(t)}, 1 \right\}  \varepsilon^{(t)}\leq \varepsilon^{(t)}-\frac a 2 \min\left\{ \frac 1 { b +\varepsilon_0} \varepsilon^{(t)}, 1 \right\} \varepsilon^{(t)} \leq  \varepsilon^{(t)}-\frac a 2 \frac 1 { b +\varepsilon_0} \varepsilon^{(t)} \varepsilon^{(t)}.
\end{eqnarray}
where we used $\varepsilon_0\geq \varepsilon^{(t)}\; \forall t>0$. Thus it holds that  
\begin{eqnarray}
\varepsilon^{(t+1)} &\leq&  \varepsilon^{(t)}-\frac a 2 \frac 1 { b +\varepsilon_0} \varepsilon^{(t)} \varepsilon^{(t+1)}
\end{eqnarray}
Deviding both sides by $\varepsilon^{(t)} \varepsilon^{(t+1)}$ yields
\begin{eqnarray}
\frac 1{ \varepsilon^{(t+1)}} &\geq&  \frac 1 {\varepsilon^{(t)}} -\frac a 2 \frac 1 { b +\varepsilon_0} 
\end{eqnarray}
Applying this bound recursively and plugging in the definition of $a,b$ concludes the proof.
\end{proof}

\subsubsection{Strongly-convex $g_i$.}
\begin{lemma}(strongly convex $g_i$)
Let $f$ be $\tfrac 1 \tau$-smooth and $g_i$  $\mu$-strongly convex with $\mu>0$. Assume the sequence $\{\sigmat\}_{t\geq 0}$ is bounded above by $\sigmasup$. Then, we can bound the suboptimality $\varepsilon^{(t)}:=\bO(\alphav^{(t)})-\bO(\alphav^\star)$ as
\begin{eqnarray*}
\bO(\alphav^{(t)})-\bO(\alphav^\star) \leq\left(1- \xi(1-\eta) \frac{\mu\tau}{c_{A}\sigmasup +\mu\tau}\right)^{|S_t|}\varepsilon^{(0)}
\end{eqnarray*}
 where $c_{A}=\max_k \|A_{[k]}\|^2$ and $|S_t|$ denotes the cardinality of the set $S_t$ which counts the number of successful updates up to iteration $t$.
 \label{lem:cocoasc}
\end{lemma}

\begin{proof}
For $\mu$-strongly convex $g_i$ with $\mu>0$ we can choose $\kappa=\hat\kappa^{(t)}$ such that $R^{(t)}\leq 0$ in Lemma \ref{lem:dec}.
That is  
\[\hat \kappa^{(t)} = \frac{\mu\tau}{c_{A}\sigma_{t} +\mu\tau}\] since
\begin{eqnarray*}
R^{(t)}&=&  \sigmat (\uv^{(t)}-\alphav^{(t)})^\top  \tilde H(\alphav^{(t)}) (\uv^{(t)}-\alphav^{(t)})- \frac {\mu (1-\kappa)}{\kappa} \|\alphav^{(t)}-\uv^{(t)}\|_2^2\\
&=&\sigmat \sum_k (A_{[k]}(\uv^{(t)}-\alphav^{(t)})_{[k]})^\top  \nabla^2 f(A\alphav) (A_{[k]}(\uv^{(t)}-\alphav^{(t)})_{[k]})- \frac {\mu (1-\kappa)}{\kappa} \|\alphav^{(t)}-\uv^{(t)}\|_2^2\\
&\leq&\frac \sigmat \tau \sum_k \|A_{[k]}(\uv^{(t)}-\alphav^{(t)})_{[k]}\|_2^2 - \frac {\mu (1-\kappa)}{\kappa} \|\alphav^{(t)}-\uv^{(t)}\|_2^2\\
&\leq& \left(\frac {c_{A}\sigma_{t} }\tau- \frac {\mu (1-\kappa)}{\kappa}\right) \|\alphav^{(t)}-\uv^{(t)}\|_2^2
\end{eqnarray*}
where  $c_{A}=\max_k \|A_{[k]}\|^2$.

Hence, by Lemma \ref{lem:f_dec},  for successful updates $(\Dav,\sigmat)$, the function decrease at iteration $t$ can be lower bounded as 
\begin{eqnarray*}
\bO(\alphat) - \bO(\alphat+\Dav)\geq \xi(1-\eta)\left[\frac{\mu\tau}{c_{A}\sigma_{t} +\mu\tau}\right]\Gap(\alphav^{(t)}).
\end{eqnarray*}
While, by construction of Algorithm \ref{alg:adaptive_cocoa}, the iterate remains unchanged over unsuccessful iterations. Let us denote the suboptimality at iteration $t$ as $\varepsilon^{(t)}=\bO(\alphat)-\bO(\alphav^\star)$. Then, using the fact that the duality gap always upper bounds the suboptimality, i.e., $\Gap(\alphat)\geq \bO(\alphat)$ we get the following recursion:
\begin{eqnarray*}
\varepsilon^{(t+1)}\leq\left(1- \xi(1-\eta)\left[\frac{\mu\tau}{c_{A}\sigma_{t} +\mu\tau}\right]\right)\varepsilon^{(t)}.
\end{eqnarray*}
Using that fact that the sequence $\{\sigmat\}_{t\geq0}$ is bounded by $\sigmasup$ we can establish the following rate of convergence
\begin{eqnarray*}
\varepsilon^{(t)}\leq\left(1- \xi(1-\eta) \left[\frac{\mu\tau}{c_{A}\sigmasup +\mu\tau}\right]\right)^{|S_t|}\varepsilon^{(0)}
\end{eqnarray*}


\end{proof}

%

\subsection{Bound on number of successful steps}

From Lemma \ref{lem:cocoasc} and Lemma \ref{lem:cocoac} in the previous section the following two lemmas follow immediately:

\begin{lemma}[Number of successful iterations]\label{l:number_of_successful_steps}
Let $f$ be $\frac 1 \tau$ smooth and  $g_i$ $\mu$-strongly convex. Assume the sequence $\{\sigma_t\}_{t\geq 0}$ is bounded above by $\sigma_\text{sup}$. Then, Algorithm \ref{alg:adaptive_cocoa} achieves a suboptimality $\bO(\alphat)-\bO(\alphav^\star)\leq\varepsilon$ after
\begin{equation}\label{eq:number_of_successful_steps}
 \frac 1 {\log(C_2^{-1})} \log\left(\frac{\varepsilon_0}{\varepsilon}\right)
\end{equation}
successful iterations where $C_2\in (0,1)$ is a constant defined as $C_2= 1- \xi(1-\eta) \frac{\mu\tau}{c_A \sigma_{\sup} +\mu\tau}$ and $\varepsilon_0$ denotes the initial suboptimaliy $\varepsilon_0=\bO(\alphav^{(0)})-\bO(\alphav^\star)$.
\end{lemma}

\begin{lemma}[Number of successful iterations]\label{l:number_of_successful_steps_c}
Let be $f$ $\frac 1 \tau$-smooth and $g_i$  convex functions with  $L$-bounded support. 
Assume the sequence $\{\sigma_t\}_{t\geq 0}$ is bounded above by $\sigma_\text{sup}$. Then, Algorithm \ref{alg:adaptive_cocoa} achieves a suboptimality $\bO(\alphat)-\bO(\alphav^\star)\leq\varepsilon$ after
\begin{equation}\label{eq:number_of_successful_steps}
C_1 \frac 1 {\varepsilon} 
\end{equation}
successful iterations where  $C_1=\frac{2\left[\tfrac 4 \tau L^2 R^2\sigmasup+\varepsilon_0\right]}{\xi (1-\eta)}$ 
\end{lemma}

\subsection{Bound on number of unsuccessful steps}
We assume there exists a $\sigma_{\text{sup}}<\infty$ such that $\sigmat\leq \sigma_\text{sup}$ for all $ t\geq 0$ (we will later in Section \ref{app_upperbound} show the existence of such a bound).
As a consequence, the algorithm  may only take a limited number of consecutive unsuccessful steps and hence the total number of unsuccessful iterations is at most a problem dependent constant times the number of successful iterations plus some additive term depending on the initialization $\sigma_0$. The following Lemma is  motivated by \citep[Corollary 5.5]{Cartis2011b}:

\begin{lemma}[Number of unsuccessful iterations]\label{l:number_of_unsuccessful_steps}
Assume the sequence $\{\sigma_t\}_{t\geq 0}$ is bounded above by $\sigma_\text{sup}$. Then, for any fixed $T\geq 0$, it holds that
\begin{equation}\label{eq:number_of_unsuccessful_steps}
|U_T|\leq\frac {1}{\log(\gamma)} \log\Big( \frac {\sigma_\text{sup}}{\sigma_0}\Big) + |S_T| 
\end{equation}
\end{lemma}
\begin{proof}
 Since it holds that $\xi\leq\frac 1 \zeta$ we have 
\[\sigma_{t+1}= \gamma \sigma_t \;\;\forall t \in U_T \]
\[\sigma_{t+1}\geq \frac 1 \gamma \sigma_t \;\;\forall t \in S_T \]

Thus, we deduce inductively
\[\sigma_0 \gamma^{-|S_T|}\gamma^{|U_T|}\leq \sigma_T\]
since $\sigma_T$ is  bounded by $ \sigma_{\text{sup}}$  we have
\[ \log\left(\frac 1 \gamma\right){|S_T|}+\log(\gamma) {|U_T|}\leq \log\Big(\frac{\sigma_{\text{sub}}}{\sigma_0}\Big)\]
Recall $\gamma>1$ and hence
\begin{equation}
|U_T|\leq\frac {1}{\log(\gamma)} \log\Big( \frac {\sigma_\text{sup}}{\sigma_0}\Big) + |S_T|
\label{eq:boundU}
\end{equation}
\end{proof}
\begin{remark}
Assume we start with a save value $\sigma_0\geq\sigma_{sup}$, then \eqref{eq:boundU} simplifies to  $|U_T|\leq 1+ |S_T| $
\end{remark}

\subsection{Final convergence result}
In order to prove the final convergence results it remains to combine Lemma \ref{lem:cocoasc} and Lemma \ref{lem:cocoac} with   Lemma \ref{l:number_of_unsuccessful_steps} in order to bound the total number of iterations $T=|S_T|+|U_T|$ required in Algorithm \ref{alg:adaptive_cocoa}  to reach a required suboptimality. Doing so results in the following two corollaries: 

%
%

\begin{corollary}[Number of successful and unsuccessful iterations]\label{l:number_of_steps_r}
Let $f$ be $\frac 1 \tau$-smooth and $g_i$ $\mu$-strongly-convex.  Assume the sequence $\{\sigma_t\}_{t\geq 0}$ is bounded above by $\sigma_\text{sup}$. Then Algorithm \ref{alg:adaptive_cocoa} reaches an accuracy $ \bO(\alphat)-\bO(\alphav^\star)\leq \varepsilon$ within a total number of 
\begin{equation}\label{eq:number_of_steps}
 \frac {1}{\log(\gamma)} \log\left( \frac {\sigma_\text{sup}}{\sigma_0}\right) + \frac 2 {\log(C_2^{-1})} \log\left(\frac{\varepsilon_0}{\varepsilon}\right)
\end{equation}
steps, where $C_2\in (0,1)$ is a constant defined as $C_2= 1- \xi(1-\eta) \frac{\mu\tau}{c_A \sigma_{\sup} +\mu\tau}$.
\end{corollary}
\begin{proof}
Combining Lemma \ref{l:number_of_unsuccessful_steps} and Lemma \ref{l:number_of_successful_steps} this result follows immediately.
\end{proof}

\begin{corollary}[Number of successful and unsuccessful iterations]\label{l:number_of_steps_r_c}
Let $f$ be $\frac 1 \tau$-smooth and $g_i$ have $L$-bounded support.  Assume the sequence $\{\sigma_t\}_{t\geq 0}$ is bounded above by $\sigma_\text{sup}$. Then, Algorithm \ref{alg:adaptive_cocoa} reaches an accuracy  $\bO(\alphat)-\bO(\alphav^\star)\leq\varepsilon$ within a total number of 
\begin{equation}\label{eq:number_of_steps}
 \frac {1}{\log(\gamma)} \log\Big( \frac {\sigma_\text{sup}}{\sigma_0}\Big)  + 2 C_1 \frac 1 {\varepsilon}
\end{equation}
steps, where $C_1>0$ is a constant defined as $C_1=\frac{2\left[\tfrac 4 \tau L^2 R^2\sigmasup+\varepsilon_0\right]}{\xi (1-\eta)}$ .
\end{corollary}

\begin{remark}
	The first term in \eqref{eq:number_of_steps} is the price we potentially pay for a bad initialization $\sigma_0$ (the number of unsuccessful steps before the first successful step happens). If we choose a safe initial value $\sigma_0\geq\sigma_{\text{sup}}$ this first term can be upper-bounded by~$1$.
\end{remark}

 
\subsection{Boundness of the sequence $\{\sigmat\}_{t\geq0}$}
\label{app_upperbound}

So far we have assumed that there exists a $\sigmasup<\infty$ that bounds the sequence $\{\sigmat\}_{t\geq0}$ generated by Algorithm \ref{alg:adaptive_cocoa}. In this section we will explicitly give such an upper bound under different conditions on $f$. To achieve this we show that if $\sigma$ is large enough  the auxiliary model builds a global upper bound on the objective function and Algorithm \ref{alg:adaptive_cocoa} must yield a successful step.

\subsubsection{quasi-self-concordant $f$}
Let us assume the function $f$ is quasi self-concordant  \cite{bach}. This assumption is not very restrictive and fulfilled by most prominent machine learning applications including logistic regression. We will first state the definition of quasi-self concordance together with some useful properties and then explicitly state an upper bound $\sigmat$.

\begin{definition}[multivariate quasi self-concordant functions] $f$ is quasi self-concordant if $\forall \wv,\vv \in \mathbb R^n$ the function $\phi(t)=f(\wv+t \vv)$ satisfies 
$|\phi{'''}(t)|\leq M_f \|\vv\|_2 \phi{''}(t)$ for some $M_f\geq 0$.
\label{def:selfcond}
\end{definition}

\begin{proposition} \citep[Proposition 1]{bach} Let $f$ be a quasi self-concordant  function with constant $M_f$, then
\[f(\wv+\vv)\leq f(\wv) + \vv^\top \nabla f(\wv)+\frac {\vv^\top \nabla^2 f(\wv)\vv}{M_f^2 \|\vv\|_2}(e^{M_f \|\vv\|_2 } - M_f\|\vv\|_2-1)\] 
and
\[\nabla^2 f (\wv+\vv)\preceq e^{M_f \|\vv\|_2}\nabla^2 f(\wv)\]
\label{propo:selfcond}
\end{proposition}

\begin{lemma}[safe bound on $\sigma_t$]
\label{lem:sigbound}
Assume $f$ be a quasi self-concordant function. Then, for every iteration $t\geq 0$ of Algorithm~\ref{alg:adaptive_cocoa} we have  $\sigmat\leq\sigma_{sup}$, where
\begin{equation}\sigma_{sup}:=2 K  \gamma  \left[\frac{e^{M_f \|\Dav\|_2 } - M_f\|\Dav\|_2-1}{M_f^2 \|\Dav\|_2}\right]
\label{eq:sigHC1}
\end{equation}
\end{lemma}

\begin{proof} To prove Lemma \ref{lem:sigbound} we show that for $\sigmat\geq \frac \sigmasup \gamma$ the model forms an upper bound on the objective, i.e. $\model_\sigmat(\Dav;\alphat)\geq \bO(\alphat+\Dav) \; \; \forall \Dav$. Hence, for every $\eta$-approximate update $\Dav$ we have $\rho_t\geq 1\geq \xi$ and hence a successful step.
In order to establish this bound on $\sigma_t$ we use Proposition \ref{propo:selfcond} which yields
\begin{align*}
f(\alphat&+\Dav)\\
&\leq f(A\alphat) + \nabla f(A\alphat)^\top A\Dav + (A\Dav)^\top \nabla^2 f(A\alphat) A\Dav\left[\frac{e^{M_f \|\Dav\|_2 } - M_f\|\Dav\|_2-1}{M_f^2 \|\Dav\|_2}\right]\\
&\overset{(i)}{\leq} f(A\alphat) + \nabla f(A\alphat)^\top A\Dav + K \sum_k (A_{[k]}\Dak)^\top \nabla^2 f(A \alphat)A_{[k]}\Dak\left[\frac{e^{M_f \|\Dav\|_2 } - M_f\|\Dav\|_2-1}{M_f^2 \|\Dav\|_2}\right]\\
&= f(A\alphat) + \nabla f(A\alphat)^\top A\Dav + K \sum_k \Dak^\top \tilde H(\alphat)\Dak\left[\frac{e^{M_f \|\Dav\|_2 } - M_f\|\Dav\|_2-1}{M_f^2 \|\Dav\|_2}\right].
\end{align*}
In $(i)$ we used Jensen's inequality for convex functions, i.e., $f(\tfrac 1 n \sum_i x_i)\leq \tfrac 1 n \sum_i f(x_i)$.\newline
Hence,
\begin{eqnarray*}
\bO(\alphat+\Dav) &=& f(\alphat+\Dav) + \sum_i g_i((\alphat+\Dav)_i)\\
&\leq& \model_{\sigmat = \frac 1 \gamma \sigmasup}(\Dav;\alphat)
\end{eqnarray*}
since any $\sigma> \frac 1 \gamma \sigmasup$ is guaranteed to yield a successful step and in every iteration $\sigma$ is at most increased by $\gamma$, hence, $\sigma_\text{sup}$ provides an upper bound on $\{\sigmat\}_{t\geq0}$

\end{proof}

\subsubsection{$f$ with Lipschitz continuous Hessian}
In this section we will show that the theoretical bound on $\sigmasup$ derived in the previous section can be refined when posing stronger assumptions of $f$. Therefore, let us pose the following two widely used assumptions on $f$:

\begin{assumption}
\label{ass:1}
Assume the Hessian of $f(A\alphav)$ and $ \tilde H(\alphav)$ agree along this direction of the step, i.e.,
\[\|(A^\top \nabla^2 f(\vv)A -  \tilde H(\alphav))\Dav_k\|\leq C \|\Dav\| \;\;\; \forall t>0 \; \; \text{and some } C>0 \]
\end{assumption}
\begin{assumption}
\label{ass:2}
Assume the Hessian of $f$ is globally Lipschitz continuous, i.e.,
\[\|\nabla^2 f(x) - \nabla^2 f(y)\|\leq L \|x-y\| \;\;\; \forall x,y\in \mathbb  R^n \; \; \text{and some } L>0 \]
\end{assumption}

Under these conditions we can show that the sequence $\{\sigma_t \}_{t\geq0}$ is bounded by $\sigmasup$ as defined in the following lemma:
\begin{lemma}
\label{lem:sf}
Assume $f$ satisfies Assumptions  \ref{ass:1} and \ref{ass:2}, then, the sequence $\{\sigma_t\}_{t\geq0}$ in Algorithm \ref{alg:adaptive_cocoa} is bounded above by
\begin{equation}\sigmasup =  \frac{\gamma}{ 2 \|\tilde H(\alphat)\| }\left(L+C+1\right).
\label{eq:sigHC2}
\end{equation}
\end{lemma}
\begin{proof}In order to prove Lemma \ref{lem:sf} we show that from $\sigma\geq\frac{1}{ 2\|\tilde H(\alphav)\| }\left(L+C+1\right) $ the Algorithm must yield a successful step. This can be achieved by proving that for $\sigma=\sigmasup$ the model globally upper bounds the objective function, i.e., $\bO(\alphat+\Dav)-\model_\sigmat(\Dav;\alphat)<0$ and hence $\rho_t\geq1>\xi$. Therefore we bound $\bO(\alphat+\Dav)-\model_\sigmat(\Dav;\alphat)$ as
\begin{eqnarray*}
\bO(\alphat+\Dav)-\model_\sigmat(\Dav;\alphat)&\leq& f(A(\alphat+\Dav)) - \left[f(A\alphat)+\nabla f(A\alphat)^\top A\Dav +\frac \sigma 2\Dav^\top H(\alphat)\Dav\right]\\
&=& f(A\alphat) + \nabla f(A\alphat)^\top A \Dav +\frac 1 2  (A\Dav)^\top \nabla^2 f(A\betav)A\alphat\\&& - \left[f(A\alphat)+\nabla f(A\alphat)^\top A\Dav -\frac \sigma 2 \Dav^\top \tilde H(\alphat)\Dav\right]\\
&=& \frac 1 2 { \Dav^\top \left[ A^\top \nabla^2 f(A\betav^{(t)})A- \sigma \tilde H(\alphat)\right]\Dav}
\end{eqnarray*}
which holds for some $\betav^{(t)}$ on the line segment $(\alphat,\alphat+\Dav)$. We continue by using Assumption \ref{ass:1} and Assumption \ref{ass:2} which yields
\begin{eqnarray*}
\bO(\alphat+\Dav)-\model_\sigmat(\Dav;\alphat)
&=&\frac 1 2 \Dav^\top \left[ A^\top \nabla^2 f(A\betav^{(t)})A-  \sigmat \tilde H(\alphat) \pm A^\top\nabla^2 f(A \alphat)A\pm \tilde H(\alphat)\right]\Dav\\
&\leq&\frac 1 2 \|  A^\top \nabla^2 f(A\betav^{(t)})A-   A^\top \nabla^2 f(A \alphat)A\|\|\Dav\|^2\\&& +\frac 1 2 \|( A^\top\nabla^2 f(A \alphat)A- \tilde H(\alphat) )\Dav\|\|\Dav\| +\frac12{(1-\sigmat)}  \|\tilde H(\alphat)\|\|\Dav\|^2\\
&\leq& \frac L 2 \|\Dav\|^2+ \frac C 2 \|\Dav\|^3 +\frac 1 2 (1-\sigmat) \|\tilde H(\alphat)\|\|\Dav\|^2\\
&\leq& \max( \|\Dav\|^3, \|\Dav\|^2)\left[ \frac L 2 + \frac C 2  +\frac 1 2 (1-\sigmat) \|\tilde H(\alphat)\|\right].
\end{eqnarray*}
We can conclude the proof by noting that the RHS is negative for $\sigmat\geq\frac \sigmasup \gamma$ as defined in \eqref{eq:sigHC2} and this guarantees $\bO(\alphat+\Dav)\leq \model_\sigmat(\Dav;\alphat)$ and hence a successful step  for any $\xi\in(0,1)$.
\end{proof}